\documentclass{article}

\usepackage{times}
\usepackage{graphicx} % more modern
\usepackage{subfigure}

\usepackage{natbib}

\usepackage{algorithm}
\usepackage{algorithmic}
\usepackage[algo2e,linesnumbered,lined,boxed,vlined]{algorithm2e}

\usepackage{amsfonts}
\usepackage{amssymb}

\usepackage{hyperref}

\usepackage[accepted]{icml2016}

\usepackage{comment}
\usepackage{amsthm}
\usepackage{mathtools}
\usepackage{tabularx}
\usepackage{multirow}

\newtheorem{lemma}{Lemma}

\def\b{\ensuremath\boldsymbol}

\usepackage{lastpage}
\usepackage{fancyhdr}
\usepackage{url}
\icmltitlerunning{Factor Analysis, Probabilistic PCA, Variational Inference, and Variational Autoencoder: Tutorial and Survey}

\begin{document}

\AddToShipoutPictureBG*{%
  \AtPageUpperLeft{%
    \setlength\unitlength{1in}%
    \hspace*{\dimexpr0.5\paperwidth\relax}
    \makebox(0,-0.75)[c]{\normalsize {\color{black} To appear as a part of an upcoming textbook on dimensionality reduction and manifold learning.}}
    }}

\twocolumn[
\icmltitle{Factor Analysis, Probabilistic Principal Component Analysis, \\Variational Inference, and Variational Autoencoder: Tutorial and Survey}

\icmlauthor{Benyamin Ghojogh}{bghojogh@uwaterloo.ca}
\icmladdress{Department of Electrical and Computer Engineering, 
\\Machine Learning Laboratory, University of Waterloo, Waterloo, ON, Canada}
\icmlauthor{Ali Ghodsi}{ali.ghodsi@uwaterloo.ca}
\icmladdress{Department of Statistics and Actuarial Science \& David R. Cheriton School of Computer Science, 
\\Data Analytics Laboratory, University of Waterloo, Waterloo, ON, Canada}
\icmlauthor{Fakhri Karray}{karray@uwaterloo.ca}
\icmladdress{Department of Electrical and Computer Engineering, 
\\Centre for Pattern Analysis and Machine Intelligence, University of Waterloo, Waterloo, ON, Canada}
\icmlauthor{Mark Crowley}{mcrowley@uwaterloo.ca}
\icmladdress{Department of Electrical and Computer Engineering, 
\\Machine Learning Laboratory, University of Waterloo, Waterloo, ON, Canada}

\icmlkeywords{Tutorial}

\vskip 0.3in
]

\begin{abstract}
This is a tutorial and survey paper on factor analysis, probabilistic Principal Component Analysis (PCA), variational inference, and Variational Autoencoder (VAE). These methods, which are tightly related, are dimensionality reduction and generative models. They assume that every data point is generated from or caused by a low-dimensional latent factor. By learning the parameters of distribution of latent space, the corresponding low-dimensional factors are found for the sake of dimensionality reduction. For their stochastic and generative behaviour, these models can also be used for generation of new data points in the data space. In this paper, we first start with variational inference where we derive the Evidence Lower Bound (ELBO) and Expectation Maximization (EM) for learning the parameters. Then, we introduce factor analysis, derive its joint and marginal distributions, and work out its EM steps. Probabilistic PCA is then explained, as a special case of factor analysis, and its closed-form solutions are derived. Finally, VAE is explained where the encoder, decoder and sampling from the latent space are introduced. Training VAE using both EM and backpropagation are explained. 
\end{abstract}

\section{Introduction}

Learning models can be divided into discriminative and generative models \cite{ng2002discriminative,bouchard2004tradeoff}. Discriminative models discriminate the classes of data for better separation of classes while the generative models learn a latent space which generates the data points. The methods introduced in this paper are generative models. 

Variational inference is a technique which finds a lower bound on the log-likelihood of data and maximizes this lower bound rather than the log-likelihood in the Maximum Likelihood Estimation (MLE). This lower bound is usually referred to as the Evidence Lower Bound (ELBO). 
Learning the parameters of latent space can be done using Expectation Maximization (EM) \cite{bishop2006pattern}.
Variational Autoencoder (VAE) \cite{kingma2014auto} implements the variational inference in an autoencoder neural network setup where the encoder and decoder model the E-step and M-step of EM, respectively. Although, VAE is usually trained using backprogatation, in practice \cite{rezende2014stochastic,hou2017deep}. 
Variational inference and VAE have had many applications in Bayesian analysis; for example, see the application of variational inference in 3D human motion analysis \cite{sminchisescu2004generative} and the application of VAE in forecasting \cite{walker2016uncertain}.

Factor analysis assumes that every data point is generated from a latent factor/variable where some noise may have been added to data in the data space. Using the EM introduced in variational inference, the ELBO is maximized and the parameters of the latent space are learned iteratively. Probabilistic PCA (PPCA), as a special case of factor analysis, restricts the noise of dimensions to be uncorrelated and assumes the variance of noise to be equal in all dimensions. This restriction makes the solution of PPCA closed-form and simpler. 

In this paper, we explain the theory and details of factor analysis, PPCA, variational inference, and VAE. The remainder of this paper is organized as follows. Section \ref{section_variational_inference}  introduces variational inference. We explain factor analysis and PPCA in Sections \ref{section_factor_analysis} and \ref{section_PPCA}, respectively. VAE is explained in Section \ref{section_VAE}. Finally, Section \ref{section_conclusion} concludes the paper.

\section*{Required Background for the Reader}

This paper assumes that the reader has general knowledge of calculus, probability, linear algebra, and basics of optimization. 

\section{Variational Inference}\label{section_variational_inference}

Consider a dataset $\{\b{x}_i\}_{i=1}^n$.
Assume that every data point $\b{x}_i \in \mathbb{R}^d$ is generated from a latent variable $\b{z}_i \in \mathbb{R}^p$. 
This latent variable has a prior distribution $\mathbb{P}(\b{z}_i)$. 
According to Bayes' rule, we have:
\begin{align}
\mathbb{P}(\b{z}_i\, |\, \b{x}_i) = \frac{\mathbb{P}(\b{x}_i\, |\, \b{z}_i)\, \mathbb{P}(\b{z}_i)}{\mathbb{P}(\b{x}_i)}.
\end{align}
Let $\mathbb{P}(\b{z}_i)$ be an arbitrary distribution denoted by $q(\b{z}_i)$. Suppose the parameter of conditional distribution of $\b{z}_i$ on $\b{x}_i$ is denoted by $\b{\theta}$; hence, $\mathbb{P}(\b{z}_i\, |\, \b{x}_i) = \mathbb{P}(\b{z}_i\, |\, \b{x}_i, \b{\theta})$. 
Therefore, we can say:
\begin{align}\label{equation_inference_Bayes2}
\mathbb{P}(\b{z}_i\, |\, \b{x}_i, \b{\theta}) = \frac{\mathbb{P}(\b{x}_i\, |\, \b{z}_i, \b{\theta})\, \mathbb{P}(\b{z}_i\, |\, \b{\theta})}{\mathbb{P}(\b{x}_i\, |\, \b{\theta})}.
\end{align}

\subsection{Evidence Lower Bound (ELBO)}

Consider the Kullback-Leibler (KL) divergence \cite{kullback1951information} between the prior probability of the latent variable and the posterior of the latent variable:
\begin{align*}
&\text{KL}\big(q(\b{z}_i)\, \|\, \mathbb{P}(\b{z}_i\, |\, \b{x}_i, \b{\theta})\big) \\
&\overset{(a)}{=} \int q(\b{z}_i) \log\big(\frac{q(\b{z}_i)}{\mathbb{P}(\b{z}_i\, |\, \b{x}_i, \b{\theta})}\big) d\b{z}_i \\
&= \int q(\b{z}_i) \big( \log(q(\b{z}_i)) - \log(\mathbb{P}(\b{z}_i\, |\, \b{x}_i, \b{\theta})) \big) d\b{z}_i \\
&\overset{(\ref{equation_inference_Bayes2})}{=} \int q(\b{z}_i) \big( \log(q(\b{z}_i)) - \log(\mathbb{P}(\b{x}_i\,|\,\b{z}_i,\b{\theta})) \\
&~~~~~~~~~~~~~~ - \log(\mathbb{P}(\b{z}_i\,|\,\b{\theta})) + \log(\mathbb{P}(\b{x}_i\,|\,\b{\theta})) \big)\, d\b{z}_i 
\end{align*}
\begin{align*}
&\overset{(b)}{=} \log(\mathbb{P}(\b{x}_i\,|\,\b{\theta})) + \int q(\b{z}_i) \big( \log(q(\b{z}_i)) \\
&~~~~~~~~~~~~~~ - \log(\mathbb{P}(\b{x}_i\,|\,\b{z}_i,\b{\theta})) - \log(\mathbb{P}(\b{z}_i\,|\,\b{\theta})) \big)\, d\b{z}_i \\
&= \log(\mathbb{P}(\b{x}_i\,|\,\b{\theta})) \\
&~~~~~~~~~~~~ + \int q(\b{z}_i) \log(\frac{q(\b{z}_i)}{\mathbb{P}(\b{x}_i\,|\,\b{z}_i,\b{\theta}) \mathbb{P}(\b{z}_i\,|\,\b{\theta})})\, d\b{z}_i \\
&= \log(\mathbb{P}(\b{x}_i\,|\,\b{\theta})) + \int q(\b{z}_i) \log(\frac{q(\b{z}_i)}{\mathbb{P}(\b{x}_i, \b{z}_i\,|\,\b{\theta})})\, d\b{z}_i \\
&= \log(\mathbb{P}(\b{x}_i\,|\,\b{\theta})) + \text{KL}\big(q(\b{z}_i)\, \|\, \mathbb{P}(\b{x}_i, \b{z}_i\, |\, \b{\theta})\big),
\end{align*}
where $(a)$ is for definition of KL divergence and $(b)$ is because $\log(\mathbb{P}(\b{x}_i\,|\,\b{\theta}))$ is independent of $\b{z}_i$ and comes out of integral and $\int d\b{z}_i = 1$. 
Hence:
\begin{equation}
\begin{aligned}
\log(\mathbb{P}(\b{x}_i\,|\,\b{\theta})) = &\,\text{KL}\big(q(\b{z}_i)\, \|\, \mathbb{P}(\b{z}_i\, |\, \b{x}_i, \b{\theta})\big) \\
&- \text{KL}\big(q(\b{z}_i)\, \|\, \mathbb{P}(\b{x}_i, \b{z}_i\, |\, \b{\theta})\big).
\end{aligned}
\end{equation}
We define the \textit{Evidence Lower Bound (ELBO)} as:
\begin{align}\label{ELBO_equation1}
\mathcal{L}(q, \b{\theta}) := - \text{KL}\big(q(\b{z}_i)\, \|\, \mathbb{P}(\b{x}_i, \b{z}_i\, |\, \b{\theta})\big).
\end{align}
So:
\begin{align*}
&\log(\mathbb{P}(\b{x}_i\,|\,\b{\theta})) = \text{KL}\big(q(\b{z}_i)\, \|\, \mathbb{P}(\b{z}_i\, |\, \b{x}_i, \b{\theta})\big) + \mathcal{L}(q, \b{\theta}).
\end{align*}
Therefore:
\begin{align}\label{ELBO_equation2}
&\mathcal{L}(q, \b{\theta}) = \log(\mathbb{P}(\b{x}_i\,|\,\b{\theta})) - \underbrace{\text{KL}\big(q(\b{z}_i)\, \|\, \mathbb{P}(\b{z}_i\, |\, \b{x}_i, \b{\theta})\big)}_{\geq 0}.
\end{align}
As the second term is negative with its minus, the ELBO is a lower bound on the log likelihood of data:
\begin{align}\label{ELBO_lower_bound}
\mathcal{L}(q, \b{\theta}) \leq \log(\mathbb{P}(\b{x}_i\,|\,\b{\theta})).
\end{align}
The likelihood $\mathbb{P}(\b{x}_i\,|\,\b{\theta})$ is also referred to as the \textit{evidence}. 
Note that this lower bound gets tight when:
\begin{align}
&\mathcal{L}(q, \b{\theta}) \approx \log(\mathbb{P}(\b{x}_i\,|\,\b{\theta})) \nonumber \\
&~~~~~~~ \implies 0 \leq \text{KL}\big(q(\b{z}_i)\, \|\, \mathbb{P}(\b{z}_i\, |\, \b{x}_i, \b{\theta})\big) \overset{\text{set}}{=} 0 \nonumber \\
&~~~~~~~ \implies q(\b{z}_i) = \mathbb{P}(\b{z}_i\, |\, \b{x}_i, \b{\theta}). \label{equation_ELBO_tight_condition}
\end{align}
This lower bound is depicted in Fig. \ref{figure_ELBO}.

\begin{figure}[!t]
\centering
\includegraphics[width=2.5in]{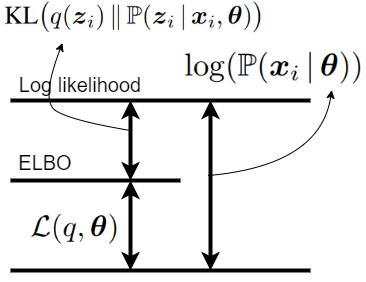}
\caption{Depiction of ELBO as the lower bound on log likelihood. The image is inspired by \cite{bishop2006pattern}.}
\label{figure_ELBO}
\end{figure}

\subsection{Expectation Maximization}

\subsubsection{Background on Expectation Maximization}

This part is taken from our previous tutorial paper \cite{ghojogh2019fitting}.
Sometimes, the data are not fully observable. For example, the data are known to be whether zero or greater than zero. 
In this case, Maximum Likelihood Expectation (MLE) cannot be directly applied as we do not have access to complete information and some data are missing.
In this case, Expectation Maximization (EM) is useful. 
The main idea of EM can be summarized in this short friendly conversation:

\begin{itshape}
-- What shall we do? Some data are missing! The log-likelihood is not known completely so MLE cannot be used.
\newline
-- Hmm, probably we can replace the missing data with something...
\newline
-- Aha! Let us replace it with its mean. 
\newline
-- You are right! We can take the mean of log-likelihood over the possible values of the missing data. Then everything in the log-likelihood will be known, and then...
\newline
-- And then we can do MLE!
\end{itshape}

EM consists of two steps which are the E-step and the M-step. 
In the E-step, the expectation of log-likelihood with respect to the missing data is calculated, in order to have a mean estimation of it. 
In the M-step, the MLE approach is used where the log-likelihood is replaced with its expectation.
These two steps are iteratively repeated until convergence of the estimated parameters.

\subsubsection{Expectation Maximization in Variational Inference}

According to MLE, we want to maximize the log-likelihood of data.
According to Eq. (\ref{ELBO_lower_bound}), maximizing the ELBO will also maximize the log-likelihood. 
The Eq. (\ref{ELBO_lower_bound}) holds for any prior distribution $q$. We want to find the best distribution to maximize the lower bound. Hence, EM for variational inference is performed iteratively as:
\begin{align}
&\text{E-step:} ~~~~~ q^{(t)} := \arg\max_q~~~~ \mathcal{L}(q, \b{\theta}^{(t-1)}), \label{equation_variational_inference_E_step} \\
&\text{M-step:} ~~~~~ \b{\theta}^{(t)} := \arg\max_{\b{\theta}}~~~~ \mathcal{L}(q^{(t)}, \b{\theta}), \label{equation_variational_inference_M_step}
\end{align}
where $t$ denotes the iteration index. 

\textbf{E-step in EM for Variational Inference}:
The E-step is:
\begin{align*}
&\max_q \mathcal{L}(q, \b{\theta}^{(t-1)}) \overset{(\ref{ELBO_equation2})}{=} \max_q \log(\mathbb{P}(\b{x}_i\,|\,\b{\theta}^{(t-1)})) \\
&~~~~~~~~~~~~~ + \max_q \big(\!-\text{KL}\big(q(\b{z}_i)\, \|\, \mathbb{P}(\b{z}_i\, |\, \b{x}_i, \b{\theta}^{(t-1)})\big) \big)\\
&= \max_q \log(\mathbb{P}(\b{x}_i\,|\,\b{\theta}^{(t-1)})) \\
&~~~~~~~~~~~~~ + \min_q \text{KL}\big(q(\b{z}_i)\, \|\, \mathbb{P}(\b{z}_i\, |\, \b{x}_i, \b{\theta}^{(t-1)})\big).
\end{align*}
The second term is always non-negative; hence, its minimum is zero:
\begin{align*}
&\text{KL}\big(q(\b{z}_i)\, \|\, \mathbb{P}(\b{z}_i\, |\, \b{x}_i, \b{\theta}^{(t-1)})\big) \overset{\text{set}}{=} 0 \\
&\implies q(\b{z}_i) = \mathbb{P}(\b{z}_i\, |\, \b{x}_i, \b{\theta}^{(t-1)}),
\end{align*}
which was already found in Eq. (\ref{equation_ELBO_tight_condition}). Thus, the E-step assigns:
\begin{align}
q^{(t)}(\b{z}_i) \gets \mathbb{P}(\b{z}_i\, |\, \b{x}_i, \b{\theta}^{(t-1)}).
\end{align}
In other words, in Fig. \ref{figure_ELBO}, it pushes the middle line toward the above line by maximizing the ELBO.

\textbf{M-step in EM for Variational Inference}:
The M-step is:
\begin{align*}
&\max_{\b{\theta}} \mathcal{L}(q^{(t)}, \b{\theta}) \overset{(\ref{ELBO_equation1})}{=} \max_{\b{\theta}} \big(\! - \text{KL}\big(q^{(t)}(\b{z}_i)\, \|\, \mathbb{P}(\b{x}_i, \b{z}_i\, |\, \b{\theta})\big) \big) \\
&\overset{(a)}{=} \max_{\b{\theta}} \Big[ -\int q^{(t)}(\b{z}_i) \log (\frac{q^{(t)}(\b{z}_i)}{\mathbb{P}(\b{x}_i, \b{z}_i\, |\, \b{\theta})})\, d\b{z}_i \Big] \\
&= \max_{\b{\theta}} \int q^{(t)}(\b{z}_i) \log (\mathbb{P}(\b{x}_i, \b{z}_i\, |\, \b{\theta}))\, d\b{z}_i \\
&~~~~~~~~~~~~~~~~~ - \max_{\b{\theta}} \int q^{(t)}(\b{z}_i) \log (q^{(t)}(\b{z}_i))\, d\b{z}_i,
\end{align*}
where $(a)$ is for definition of KL divergence. The second term is constant w.r.t. $\b{\theta}$. Hence:
\begin{align*}
&\max_{\b{\theta}} \mathcal{L}(q^{(t)}, \b{\theta}) = \max_{\b{\theta}} \int q^{(t)}(\b{z}_i) \log (\mathbb{P}(\b{x}_i, \b{z}_i\, |\, \b{\theta}))\, d\b{z}_i \\
&\overset{(a)}{=} \max_{\b{\theta}} \mathbb{E}_{\sim q^{(t)}(\b{z}_i)} \big[\log \mathbb{P}(\b{x}_i, \b{z}_i\, |\, \b{\theta})\big],
\end{align*}
where $(a)$ is because of definition of expectation. Thus, the M-step assigns:
\begin{align}
\b{\theta}^{(t)} \gets \arg \max_{\b{\theta}}~ \mathbb{E}_{\sim q^{(t)}(\b{z}_i)} \big[\log \mathbb{P}(\b{x}_i, \b{z}_i\, |\, \b{\theta})\big].
\end{align}
In other words, in Fig. \ref{figure_ELBO}, it pushes the above line higher.
The E-step and M-step together somehow play a game where the E-step tries to reach the middle line (or the ELBO) to the log-likelihood and the M-step tries to increase the above line (or the log-likelihood). This procedure is done repeatedly so the two steps help each other improve to higher values. 

To summarize, the EM in variational inference is:
\begin{align}
& q^{(t)}(\b{z}_i) \gets \mathbb{P}(\b{z}_i\, |\, \b{x}_i, \b{\theta}^{(t-1)}), \label{equation_E_step_variationalInference} \\
& \b{\theta}^{(t)} \gets \arg \max_{\b{\theta}}~ \mathbb{E}_{\sim q^{(t)}(\b{z}_i)} \big[\log \mathbb{P}(\b{x}_i, \b{z}_i\, |\, \b{\theta})\big]. \label{equation_M_step_variationalInference}
\end{align}
It is noteworthy that, in variational inference, sometimes, the parameter $\b{\theta}$ is absorbed into the latent variable $\b{z}_i$. According to the chain rule, we have:
\begin{align*}
\mathbb{P}(\b{x}_i, \b{z}_i, \b{\theta}) = \mathbb{P}(\b{x}_i\, |\, \b{z}_i, \b{\theta})\, \mathbb{P}(\b{z}_i\, |\, \b{\theta})\, \mathbb{P}(\b{\theta}).
\end{align*}
Considering the term $\mathbb{P}(\b{z}_i\, |\, \b{\theta})\, \mathbb{P}(\b{\theta})$ as one probability term, we have:
\begin{align*}
\mathbb{P}(\b{x}_i, \b{z}_i) = \mathbb{P}(\b{x}_i\, |\, \b{z}_i)\, \mathbb{P}(\b{z}_i),
\end{align*}
where the parameter $\b{\theta}$ disappears because of absorption. 

\section{Factor Analysis}\label{section_factor_analysis}

\subsection{Background on Marginal Multivariate Gaussian Distribution}

Consider two random variables $\b{x}_i \in \mathbb{R}^d$ and $\b{z}_i \in \mathbb{R}^p$ and let $\b{y}_i := [\b{x}_i^\top, \b{z}_i^\top]^\top \in \mathbb{R}^{d + p}$. Assume that $\b{x}_i$ and $\b{z}_i$ are jointly multivariate Gaussian; hence, the variable $\b{y}_i$ has a multivariate Gaussian distribution, i.e., $\b{y}_i \sim \mathcal{N}(\b{\mu}_y, \b{\Sigma}_y)$.  The mean and covariance can be decomposed as:
\begin{align}
&\b{\mu}_y =  [\b{\mu}^\top, \b{\mu}_0^\top]^\top \in \mathbb{R}^{d + p}, \\
&\b{\Sigma}_y = 
\begin{bmatrix}
\b{\Sigma}_{11} & \b{\Sigma}_{12} \\
\b{\Sigma}_{21} & \b{\Sigma}_{22}
\end{bmatrix}
\in \mathbb{R}^{(d + p) \times (d + p)}, \label{equation_factor_analysis_Sigma_y}
\end{align}
where $\b{\mu} \in \mathbb{R}^d$, $\b{\mu}_0 \in \mathbb{R}^p$, $\b{\Sigma}_{11} \in \mathbb{R}^{d \times d}$, $\b{\Sigma}_{22} \in \mathbb{R}^{p \times p}$, $\b{\Sigma}_{12} \in \mathbb{R}^{d \times p}$, and $\b{\Sigma}_{21} = \b{\Sigma}_{12}^\top \in \mathbb{R}^{p \times d}$. 

It can be shown that the marginal distributions for $\b{x}_i$ and $\b{z}_i$ are Gaussian distributions where $\mathbb{E}[\b{x}_i] = \b{\mu}$ and $\mathbb{E}[\b{z}_i] = \b{\mu}_0$ \cite{ng2018cs229}. 
The covariance matrix of the joint distribution can be simplified as \cite{ng2018cs229}:
\begin{align}
&\b{\Sigma} = 
\mathbb{E}[(\b{y}_i - \b{\mu}_y) (\b{y}_i - \b{\mu}_y)^\top] \nonumber \\
&= \mathbb{E}\Bigg[
\begin{bmatrix}
\b{x}_i - \b{\mu} \\
\b{z}_i - \b{\mu}_0 
\end{bmatrix}
\begin{bmatrix}
\b{x}_i - \b{\mu} \\
\b{z}_i - \b{\mu}_0 
\end{bmatrix}^\top
\Bigg] \nonumber \\
&= \mathbb{E}\Bigg[
\begin{bmatrix}
(\b{x}_i - \b{\mu})(\b{x}_i - \b{\mu})^\top, (\b{x}_i - \b{\mu})(\b{z}_i - \b{\mu}_0)^\top \\
(\b{z}_i - \b{\mu}_0)(\b{x}_i - \b{\mu})^\top, (\b{z}_i - \b{\mu}_0)(\b{z}_i - \b{\mu}_0)^\top
\end{bmatrix}
\Bigg].
\end{align}
This shows that the marginal distributions are:
\begin{align}
&\b{x}_i \sim \mathcal{N}(\b{\mu}, \b{\Sigma}_{11}), \\
&\b{z}_i \sim \mathcal{N}(\b{\mu}_0, \b{\Sigma}_{22}).
\end{align}
According to the definition of the multivariate Gaussian distribution, the conditional distribution is also a Gaussian distribution, i.e., $\b{x}_i | \b{z}_i \sim \mathcal{N}(\b{\mu}_{x|z}, \b{\Sigma}_{x|z})$ where \cite{ng2018cs229}:
\begin{align}
&\mathbb{R}^{d} \ni \b{\mu}_{x|z} := \b{\mu} + \b{\Sigma}_{12}\b{\Sigma}_{22}^{-1} (\b{z}_i - \b{\mu}_0), \\
&\mathbb{R}^{d \times d} \ni \b{\Sigma}_{x|z} := \b{\Sigma}_{11} - \b{\Sigma}_{12}\b{\Sigma}_{22}^{-1} \b{\Sigma}_{21},
\end{align}
and likewise for $\b{z}_i | \b{x}_i \sim \mathcal{N}(\b{\mu}_{z|x}, \b{\Sigma}_{z|x})$:
\begin{align}
&\mathbb{R}^{p} \ni \b{\mu}_{z|x} := \b{\mu}_0 + \b{\Sigma}_{21}\b{\Sigma}_{11}^{-1} (\b{x}_i - \b{\mu}), \label{equation_factor_analysis_z_given_x_mu} \\
&\mathbb{R}^{p \times p} \ni \b{\Sigma}_{z|x} := \b{\Sigma}_{22} - \b{\Sigma}_{21}\b{\Sigma}_{11}^{-1} \b{\Sigma}_{12}. \label{equation_factor_analysis_z_given_x_Sigma} 
\end{align}

\subsection{Main Idea of Factor Analysis}\label{section_factor_analysis_main_idea}

Factor analysis \cite{fruchter1954introduction,cattell1965biometrics,harman1976modern,child1990essentials} is one of the simplest and most fundamental generative models. Although its theoretical derivations are a little complicated but its main idea is very simple. 
Factor analysis assumes that every data point $\b{x}_i \in \mathbb{R}^d$ is generated from a latent variable $\b{z}_i \in \mathbb{R}^p$. 
The latent variable is also referred to as the latent factor; hence, the name of factor analysis comes from the fact that it analyzes the latent factors.

In factor analysis, we assume that the data point $\b{x}_i$ is obtained through the following steps: (1) by linear projection of the $p$-dimensional $\b{z}_i$ onto a $d$-dimensional space by projection matrix $\b{\Lambda} \in \mathbb{R}^{d \times p}$, then (2) applying some linear translation, and finally (3) adding a Gaussian noise $\b{\epsilon} \in \mathbb{R}^d$ with covariance matrix $\b{\Psi} \in \mathbb{R}^{d \times d}$. Note that as the noises in different dimensions are independent, the covariance matrix 
$\b{\Psi}$ is diagonal. 
Factor analysis can be illustrated as a graphical model \cite{ghahramani1996algorithm} where the visible data variable is conditioned on the latent variable and the noise random variable.. Figure \ref{figure_factor_analysis} shows this graphical model.

\begin{figure}[!t]
\centering
\includegraphics[width=1.1in]{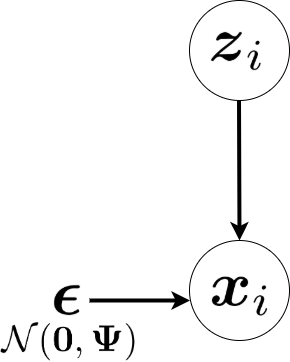}
\caption{The graphical model for factor analysis. The image is inspired by \cite{ghahramani1996algorithm}.}
\label{figure_factor_analysis}
\end{figure}

\subsection{The Factor Analysis Model}

For simplicity, the prior distribution of the latent variable can be assumed to be a multivariate Gaussian distribution:
\begin{align}
&\mathbb{P}(\b{z}_i) = \mathcal{N}(\b{z}_i\, |\, \b{\mu}_0, \b{\Sigma}_0) \nonumber \\
&= \frac{1}{\sqrt{(2\pi)^p |\b{\Sigma}_0|}} \exp\Big(\!\!- \frac{(\b{z}_i - \b{\mu}_0)^\top \b{\Sigma}_0^{-1} (\b{z}_i - \b{\mu}_0)}{2}\Big),
\end{align}
where $\b{\mu}_0 \in \mathbb{R}^p$ and $\b{\Sigma}_0 \in \mathbb{R}^{p \times p}$ are the mean and the covariance matrix of $\b{z}_i$ and $|.|$ is the determinant of matrix.
As was explain in Section \ref{section_factor_analysis_main_idea}, $\b{x}_i$ is obtained through (1) the linear projection of $\b{z}_i$ by $\b{\Lambda} \in \mathbb{R}^{d \times p}$, (2) applying some linear translation, and (3) adding a Gaussian noise $\b{\epsilon} \in \mathbb{R}^d$ with covariance $\b{\Psi} \in \mathbb{R}^{d \times d}$.
Hence, the data point $\b{x}_i$ has a conditional multivariate Gaussian distribution given the latent variable; its conditional likelihood is:
\begin{align}\label{equation_factor_analysis_likelihood_x_given_z}
\mathbb{P}(\b{x}_i\, |\, \b{z}_i) = \mathbb{P}(\b{x}_i\, |\, \b{z}_i, \b{\Lambda}, \b{\mu}, \b{\Psi}) = \mathcal{N}(\b{\Lambda} \b{z}_i + \b{\mu}, \b{\Psi}),
\end{align}
where $\b{\mu}$, which is the translation vector, is the mean of data $\{\b{x}_i\}_{i=1}^n$:
\begin{align}
\mathbb{R}^d \ni \b{\mu} := \frac{1}{n} \sum_{i=1}^n \b{x}_i.
\end{align}

The marginal distribution of $\b{x}_i$ is:
\begin{align}
&\mathbb{P}(\b{x}_i) = \int \mathbb{P}(\b{x}_i\, |\, \b{z}_i)\, \mathbb{P}(\b{z}_i)\, d\b{z}_i \implies \nonumber \\
&\mathbb{P}(\b{x}_i\, |\, \b{\Lambda}, \b{\mu}, \b{\Psi}) \nonumber \\
&~~~~~ = \int \mathbb{P}(\b{x}_i\, |\, \b{z}_i, \b{\Lambda}, \b{\mu}, \b{\Psi})\, \mathbb{P}(\b{z}_i\, |\, \b{\mu}_0, \b{\Sigma}_0)\, d\b{z}_i \nonumber \\
&~~~~~ \overset{(a)}{=} \mathcal{N}(\b{\Lambda} \b{\mu}_0 + \b{\mu}, \b{\Psi} + \b{\Lambda} \b{\Sigma}_0 \b{\Lambda}^\top) \\
&~~~~~ = \mathcal{N}(\widehat{\b{\mu}}, \b{\Psi} + \widehat{\b{\Lambda}} \widehat{\b{\Lambda}}^\top),
\end{align}
where $\mathbb{R}^d \ni \widehat{\b{\mu}} := \b{\Lambda} \b{\mu}_0 + \b{\mu}$, $\mathbb{R}^{d \times d} \ni \widehat{\b{\Lambda}} := \b{\Lambda} \b{\Sigma}_0^{(1/2)}$, and $(a)$ is because mean is linear and variance is quadratic so the mean and variance of projection are applied linearly and quadratically, respectively. 

As the mean $\widehat{\b{\mu}}$ and covariance $\widehat{\b{\Lambda}}$ are needed to be learned, we can absorb $\b{\mu}_0$ and $\b{\Sigma}_0$ into $\b{\mu}$ and $\b{\Lambda}$ and assume that $\b{\mu}_0 = \b{0}$ and $\b{\Sigma}_0 = \b{I}$.

In summary, factor analysis assumes every data point $\b{x}_i \in \mathbb{R}^d$ is obtained by projecting a latent variable $\b{z}_i \in \mathbb{R}^p$ onto a $d$-dimensional space by projection matrix $\b{\Lambda} \in \mathbb{R}^{d \times p}$ and translating it by $\b{\mu} \in \mathbb{R}^d$ and finally adding some Gaussian noise $\b{\epsilon} \in \mathbb{R}^d$ (whose dimensions are independent) as: 
\begin{align}
&\b{x}_i := \b{\Lambda} \b{z}_i + \b{\mu} + \b{\epsilon}, \label{equation_x_z_epsilon} \\
&\mathbb{P}(\b{z}_i) = \mathcal{N}(\b{0}, \b{I}), \label{equation_factor_analysis_prior_z} \\
&\mathbb{P}(\b{\epsilon}) = \mathcal{N}(\b{0}, \b{\Psi}). \label{equation_factor_analysis_prior_epsilon}
\end{align}

\subsection{The Joint and Marginal Distributions in Factor Analysis}

The joint distribution of $\b{x}_i$ and $\b{z}_i$ is:
\begin{align}
\b{y}_i := 
\begin{bmatrix}
\b{x}_i \\
\b{z}_i
\end{bmatrix} 
\sim \mathcal{N}(\b{\mu}_y, \b{\Sigma}_y).
\end{align}
The expectation of $\b{x}_i$ is:
\begin{align}\label{equation_factor_analysis_expectation_x}
\mathbb{E}[\b{x}_i] \overset{(\ref{equation_x_z_epsilon})}{=} \mathbb{E}[\b{\Lambda} \b{z}_i + \b{\mu} + \b{\epsilon}] = \b{\Lambda} \mathbb{E}[\b{z}_i] + \b{\mu} + \mathbb{E}[\b{\epsilon}] \overset{(a)}{=} \b{\mu},
\end{align}
where $(a)$ is because of Eqs. (\ref{equation_factor_analysis_prior_z}) and (\ref{equation_factor_analysis_prior_epsilon}). 
Hence:
\begin{align}
\b{\mu}_y := 
\begin{bmatrix}
\b{\mu}_x \\
\b{\mu}_z
\end{bmatrix} 
\overset{(a)}{=}
\begin{bmatrix}
\b{\mu} \\
\b{0}
\end{bmatrix},
\end{align}
where $(a)$ is because of Eqs. (\ref{equation_factor_analysis_prior_z}) and (\ref{equation_factor_analysis_expectation_x}). 
Consider Eq. (\ref{equation_factor_analysis_Sigma_y}). 
According to Eq. (\ref{equation_factor_analysis_prior_z}), we have $\b{\Sigma}_{22} = \b{\Sigma}_z = \b{I}$. 
According to Eq. (\ref{equation_x_z_epsilon}), we have:
\begin{align}
&\b{\Sigma}_{11} = \b{\Sigma}_x = \mathbb{E}[(\b{x}_i - \b{\mu})(\b{x}_i - \b{\mu})^\top] \nonumber \\
&= \mathbb{E}[(\b{\Lambda} \b{z}_i + \b{\mu} + \b{\epsilon} - \b{\mu})(\b{\Lambda} \b{z}_i + \b{\mu} + \b{\epsilon} - \b{\mu})^\top] \nonumber \\
&= \mathbb{E}[\b{\Lambda}\b{z}_i\b{z}_i^\top\b{\Lambda}^\top + \b{\epsilon}\b{z}_i^\top\b{\Lambda}^\top + \b{\Lambda}\b{z}_i\b{\epsilon}^\top + \b{\epsilon}\b{\epsilon}^\top] \nonumber \\
&= \b{\Lambda}\mathbb{E}[\b{z}_i\b{z}_i^\top]\b{\Lambda}^\top + \mathbb{E}[\b{\epsilon}]\mathbb{E}[\b{z}_i]^\top\b{\Lambda}^\top + \b{\Lambda}\mathbb{E}[\b{z}_i]\mathbb{E}[\b{\epsilon}]^\top + \mathbb{E}[\b{\epsilon}\b{\epsilon}^\top] \nonumber \\
&\overset{(a)}{=} \b{\Lambda}\b{I}\b{\Lambda}^\top + \b{0} + \b{0} + \b{\Psi} = \b{\Lambda}\b{\Lambda}^\top + \b{\Psi},
\end{align}
where $(a)$ is because of Eqs. (\ref{equation_factor_analysis_prior_z}) and (\ref{equation_factor_analysis_prior_epsilon}).
Moreover, we have:
\begin{align}
&\b{\Sigma}_{12} = \b{\Sigma}_{xz} = \mathbb{E}[(\b{x}_i - \b{\mu})(\b{z}_i - \b{\mu}_0)^\top] \nonumber \\
&\overset{(a)}{=} \mathbb{E}[(\b{\Lambda} \b{z}_i + \b{\mu} + \b{\epsilon} - \b{\mu})(\b{z}_i - \b{0})^\top] \nonumber \\
&\overset{(b)}{=} \b{\Lambda}\mathbb{E}[\b{z}_i\b{z}_i^\top] + \mathbb{E}[\b{\epsilon}]\mathbb{E}[\b{z}_i^\top] = \b{\Lambda}\b{I} + (\b{0} \b{0}^\top) = \b{\Lambda},
\end{align}
where $(a)$ is because of Eqs. (\ref{equation_x_z_epsilon}) and (\ref{equation_factor_analysis_prior_z}) and $(b)$ is because $\b{z}_i$ and $\b{\epsilon}$ are independent. We also have $\b{\Sigma}_{21} = \b{\Sigma}_{12}^\top = \b{\Lambda}^\top$. Therefore:
\begin{align}
\begin{bmatrix}
\b{x}_{i} \\
\b{z}_{i}
\end{bmatrix}
\sim
\mathcal{N}\Bigg(
\begin{bmatrix}
\b{\mu} \\
\b{0} 
\end{bmatrix}
,
\begin{bmatrix}
\b{\Lambda}\b{\Lambda}^\top + \b{\Psi} & \b{\Lambda} \\
\b{\Lambda}^\top & \b{I}
\end{bmatrix}
\Bigg).
\end{align}
Hence, the marginal distribution of data point $\b{x}_i$ is:
\begin{align}
&\mathbb{P}(\b{x}_i) = \mathbb{P}(\b{x}_i\, |\, \b{\Lambda}, \b{\mu}, \b{\Psi}) = \mathcal{N}(\b{\mu}, \b{\Lambda}\b{\Lambda}^\top + \b{\Psi}). \label{equation_factor_analysis_prior_of_data}
\end{align}
According to Eqs. (\ref{equation_factor_analysis_z_given_x_mu}) and (\ref{equation_factor_analysis_z_given_x_Sigma}), the posterior or the conditional distribution of latent variable given data is:
\begin{equation}\label{equation_factor_analysis_z_given_x}
\begin{aligned}
q(\b{z}_i) \overset{(\ref{equation_E_step_variationalInference})}{=} \mathbb{P}(\b{z}_i\, |\, \b{x}_i) &= \mathbb{P}(\b{z}_i\, |\, \b{x}_i, \b{\Lambda}, \b{\mu}, \b{\Psi}) \\
&= \mathcal{N}(\b{\mu}_{z|x}, \b{\Sigma}_{z|x}),
\end{aligned}
\end{equation}
where:
\begin{align}
&\mathbb{R}^{p} \ni \b{\mu}_{z|x} := \b{\Lambda}^\top (\b{\Lambda}\b{\Lambda}^\top + \b{\Psi})^{-1} (\b{x}_i - \b{\mu}), \label{equation_factor_analysis_z_given_x_mu_2} \\
&\mathbb{R}^{p \times p} \ni \b{\Sigma}_{z|x} := \b{I} - \b{\Lambda}^\top (\b{\Lambda}\b{\Lambda}^\top + \b{\Psi})^{-1} \b{\Lambda}. \label{equation_factor_analysis_z_given_x_Sigma_2} 
\end{align}
Recall that the conditional distribution of data given the latent variable, i.e. $\mathbb{P}(\b{x}_i\, |\, \b{z}_i)$, was introduced in Eq. (\ref{equation_factor_analysis_likelihood_x_given_z}).

If data $\{\b{x}_i\}_{i=1}^n$ are centered, i.e. $\b{\mu} = \b{0}$, the marginal of data, Eq. (\ref{equation_factor_analysis_prior_of_data}), and the likelihood of data, Eq. (\ref{equation_factor_analysis_likelihood_x_given_z}), become:
\begin{align}
&\mathbb{P}(\b{x}_i\, |\, \b{\Lambda}, \b{\Psi}) = \mathcal{N}(\b{0}, \b{\Psi} + \b{\Lambda}\b{\Lambda}^\top), \label{equation_factor_analysis_prior_of_data_centered} \\
&\mathbb{P}(\b{x}_i\, |\, \b{z}_i, \b{\Lambda}, \b{\Psi}) = \mathcal{N}(\b{\Lambda} \b{z}_i, \b{\Psi}), \label{equation_factor_analysis_likelihood_centered}
\end{align}
respectively. In some works, people center the data as a pre-processing to factor analysis.

\subsection{Expectation Maximization in Factor Analysis}

\subsubsection{Maximization of Joint Likelihood}

In factor analysis, the parameter $\b{\theta}$ of variational inference is the two parameters $\b{\Lambda}$ and $\b{\Psi}$.
As we have in Eq. (\ref{equation_M_step_variationalInference}), consider the maximization of joint likelihood, which reduces to the likelihood of data, over all $n$ data points:
\begin{align*}
& \max_{\b{\Lambda}, \b{\Psi}}~ \sum_{i=1}^n \mathbb{E}_{\sim q^{(t)}(\b{z}_i)} \big[\log \mathbb{P}(\b{x}_i, \b{z}_i\, |\, \b{\Lambda}, \b{\Psi})\big] \\
&\overset{(a)}{=} \max_{\b{\Lambda}, \b{\Psi}}~ \sum_{i=1}^n \Big( \mathbb{E}_{\sim q^{(t)}(\b{z}_i)} \big[\log \mathbb{P}(\b{x}_i\, |\, \b{z}_i, \b{\Lambda}, \b{\Psi})\big] \\
&~~~~~~~~~~~~~~~~~~ + \mathbb{E}_{\sim q^{(t)}(\b{z}_i)} \big[\log \mathbb{P}(\b{z}_i)\big] \Big), \\
&\overset{(b)}{=} \max_{\b{\Lambda}, \b{\Psi}}~ \sum_{i=1}^n \mathbb{E}_{\sim q^{(t)}(\b{z}_i)} \big[\log \mathbb{P}(\b{x}_i\, |\, \b{z}_i, \b{\Lambda}, \b{\Psi})\big]  \\
&\overset{(\ref{equation_factor_analysis_likelihood_x_given_z})}{=} \max_{\b{\Lambda}, \b{\Psi}}~ \sum_{i=1}^n \mathbb{E}_{\sim q^{(t)}(\b{z}_i)} \big[\log \mathcal{N}(\b{\Lambda} \b{z}_i + \b{\mu}, \b{\Psi})\big]  
\end{align*}
\begin{align*}
&= \max_{\b{\Lambda}, \b{\Psi}}~ \sum_{i=1}^n \mathbb{E}_{\sim q^{(t)}(\b{z}_i)} \bigg[ \log\big( \frac{1}{(2\pi)^{p/2} |\b{\Psi}|^{1/2}} \\
&~~~~~\exp\Big(\!\!- \frac{(\b{x}_i - \b{\Lambda} \b{z}_i - \b{\mu})^\top \b{\Psi}^{-1} (\b{x}_i - \b{\Lambda} \b{z}_i - \b{\mu})}{2}\Big) \big) \bigg]
\end{align*}
\begin{align}
&= \max_{\b{\Lambda}, \b{\Psi}}~ \Big( \underbrace{-\frac{d\,n}{2} \log(2\pi)}_\text{constant} - \frac{n}{2} \log |\b{\Psi}| \nonumber \\
&~~~~ - \sum_{i=1}^n \mathbb{E}_{\sim q^{(t)}(\b{z}_i)} \big[\frac{1}{2} (\b{x}_i - \b{\Lambda}\b{z}_i - \b{\mu})^\top \b{\Psi}^{-1} \nonumber \\
&~~~~~~~~~~~~~~~~~~~~~~~~~~~~~~~~~~~~~~~~~ (\b{x}_i - \b{\Lambda}\b{z}_i - \b{\mu}) \big] \Big) \label{equation_factor_analysis_EM_likelihood}
\end{align}
where $(a)$ is because of the chain rule $\mathbb{P}(\b{x}_i, \b{z}_i\, |\, \b{\Lambda}, \b{\Psi}) = \mathbb{P}(\b{x}_i\, |\, \b{z}_i, \b{\Lambda}, \b{\Psi})\, \mathbb{P}(\b{z}_i)$, and $(b)$ is because the second term is zero because of zero mean of prior of $\b{z}_i$ (see Eq. (\ref{equation_factor_analysis_prior_z})).

\subsubsection{The E-Step in EM for Factor Analysis}

As we will see later in the M-step of EM, we will have two expectation terms which need to be computed in the E-step. These expectations, which are over the $q(\b{z}_i)$ distribution, are $\mathbb{E}_{\sim q^{(t)}(\b{z}_i)}[\b{z}_i]$ and $\mathbb{E}_{\sim q^{(t)}(\b{z}_i)}[\b{z}_i \b{z}_i^\top]$.
According to Eq. (\ref{equation_E_step_variationalInference}), we have $q(\b{z}_i) = \mathbb{P}(\b{z}_i\, |\, \b{x}_i)$. Therefore, according to Eqs. (\ref{equation_factor_analysis_z_given_x}), (\ref{equation_factor_analysis_z_given_x_mu_2}), and (\ref{equation_factor_analysis_z_given_x_Sigma_2}), we have:
\begin{align}
&\mathbb{E}_{\sim q^{(t)}(\b{z}_i)}[\b{z}_i] = \b{\mu}_{z|x} := \b{\Lambda}^\top (\b{\Lambda}\b{\Lambda}^\top + \b{\Psi})^{-1} (\b{x}_i - \b{\mu}), \label{equation_factor_analysis_z_given_x_mu_3} \\
&\mathbb{E}_{\sim q^{(t)}(\b{z}_i)}[\b{z}_i \b{z}_i^\top] = \b{\Sigma}_{z|x} := \b{I} - \b{\Lambda}^\top (\b{\Lambda}\b{\Lambda}^\top + \b{\Psi})^{-1} \b{\Lambda}. \label{equation_factor_analysis_z_given_x_Sigma_3} 
\end{align}

\subsubsection{The M-Step in EM for Factor Analysis}

We have two variables $\b{\Lambda}$ and $\b{\Psi}$ so we solve the maximization w.r.t. these variables. 

\textbf{Finding parameter} $\b{\Lambda}$:
\begin{align*}
&\mathbb{R}^{d \times p} \ni \frac{\partial\, \text{Eq.}\, (\ref{equation_factor_analysis_EM_likelihood})}{\partial \b{\Lambda}} \\
&= - \sum_{i=1}^n \frac{\partial}{\partial \b{\Lambda}} \mathbb{E}_{\sim q^{(t)}(\b{z}_i)} \Big[\frac{1}{2} \textbf{tr}(\b{z}_i^\top \b{\Lambda}^\top \b{\Psi}^{-1} \b{\Lambda} \b{z}_i) \\
&~~~~~~~~~~~~~~~~~~~~~~~~~~~~~~~~ - \textbf{tr}(\b{z}_i^\top \b{\Lambda}^\top \b{\Psi}^{-1} (\b{x}_i - \b{\mu})) \Big] 
\end{align*}
\begin{align*}
&\overset{(a)}{=} - \sum_{i=1}^n \frac{\partial}{\partial \b{\Lambda}} \mathbb{E}_{\sim q^{(t)}(\b{z}_i)} \Big[\frac{1}{2} \textbf{tr}(\b{\Lambda}^\top \b{\Psi}^{-1} \b{\Lambda} \b{z}_i \b{z}_i^\top) \\
&~~~~~~~~~~~~~~~~~~~~~~~~~~~~~~~~ - \textbf{tr}(\b{\Lambda}^\top \b{\Psi}^{-1} (\b{x}_i - \b{\mu}) \b{z}_i^\top) \Big] \\
&= - \sum_{i=1}^n \mathbb{E}_{\sim q^{(t)}(\b{z}_i)} \Big[\b{\Psi}^{-1} \b{\Lambda} \b{z}_i \b{z}_i^\top - \b{\Psi}^{-1} (\b{x}_i - \b{\mu}) \b{z}_i^\top \Big] 
\end{align*}
\begin{align*}
&= - \sum_{i=1}^n \Big[\b{\Psi}^{-1} \b{\Lambda} \mathbb{E}_{\sim q^{(t)}(\b{z}_i)}[\b{z}_i \b{z}_i^\top] \\
&~~~~~~~~~~~~~~~~~~~~~~ - \b{\Psi}^{-1} (\b{x}_i - \b{\mu}) \mathbb{E}_{\sim q^{(t)}(\b{z}_i)}[\b{z}_i]^\top \Big], 
\end{align*}
where $(a)$ is because of the cyclic property of trace.
Setting this derivative to zero gives us the optimum $\b{\Lambda}$:
\begin{align}
&\sum_{i=1}^n \b{\Psi}^{-1} \b{\Lambda} \mathbb{E}_{\sim q^{(t)}(\b{z}_i)}[\b{z}_i \b{z}_i^\top] \nonumber \\
&~~~~~~~~~~~~~~~~~ = \sum_{i=1}^n \b{\Psi}^{-1} (\b{x}_i - \b{\mu}) \mathbb{E}_{\sim q^{(t)}(\b{z}_i)}[\b{z}_i]^\top \nonumber \\
&\implies \b{\Lambda} = \Big(\sum_{i=1}^n \b{\Psi}^{-1} (\b{x}_i - \b{\mu}) \mathbb{E}_{\sim q^{(t)}(\b{z}_i)}[\b{z}_i]^\top\Big) \nonumber \\
&~~~~~~~~~~~~~~~~~~~~~~~~~~ \Big( \sum_{i=1}^n \b{\Psi}^{-1} \b{\Lambda} \mathbb{E}_{\sim q^{(t)}(\b{z}_i)}[\b{z}_i \b{z}_i^\top] \Big)^{-1}. \label{equation_factor_analysis_Lambda}
\end{align}

\textbf{Finding parameter} $\b{\Psi}$:
Now, consider maximization w.r.t $\b{\Psi}$.
We restate \text{Eq.}\, (\ref{equation_factor_analysis_EM_likelihood}) as \cite{paola2018lecture}:
\begin{align}\label{equation_factor_analysis_EM_likelihood_with_S}
&\max_{\b{\Lambda}, \b{\Psi}}~ \big( \underbrace{-\frac{d\,n}{2} \log(2\pi)}_\text{constant} - \frac{n}{2} \log |\b{\Psi}| - \frac{n}{2} \textbf{tr}(\b{\Psi}^{-1} \b{S}) \big),
\end{align}
where $\b{S} \in \mathbb{R}^{d \times d}$ is a sample covariance matrix defined as:
\begin{align}
&\b{S}\! := \!\frac{1}{n} \sum_{i=1}^n\! \mathbb{E}_{\sim q^{(t)}(\b{z}_i)}\! \big[ (\b{x}_i - \b{\Lambda}\b{z}_i - \b{\mu})  (\b{x}_i - \b{\Lambda}\b{z}_i - \b{\mu})^\top \big] \nonumber \\
&= \frac{1}{n} \sum_{i=1}^n \Big( (\b{x}_i - \b{\mu})(\b{x}_i - \b{\mu})^\top \nonumber \\
&~~ - 2 \b{\Lambda} \mathbb{E}_{\sim q^{(t)}(\b{z}_i)}\! [\b{z}_i] (\b{x}_i - \b{\mu})^\top + \b{\Lambda} \mathbb{E}_{\sim q^{(t)}(\b{z}_i)}\! [\b{z} \b{z}^\top] \b{\Lambda}^\top \Big). \label{equation_factor_analysis_S}
\end{align}
The maximization results in \cite{paola2018lecture}:
\begin{align*}
&\mathbb{R}^{d \times d} \ni \frac{\partial\, \text{Eq.}\, (\ref{equation_factor_analysis_EM_likelihood_with_S})}{\partial \b{\Psi}^{-1}} = \frac{n}{2} \b{\Psi} - \frac{n}{2} \b{S} \overset{\text{set}}{=} \b{0} \implies \b{\Psi} = \b{S}.
\end{align*}
Note that as the dimensions of noise $\b{\epsilon} \in \mathbb{R}^d$ are independent, the covariance matrix of noise, $\b{\Psi}$, is a diagonal matrix. Hence:
\begin{equation}\label{equation_factor_analysis_Psi}
\begin{aligned}
&\b{\Psi} = \textbf{diag}(\b{S}) \overset{(\ref{equation_factor_analysis_S})}{=} \frac{1}{n} \textbf{diag}\bigg( \sum_{i=1}^n \Big[ (\b{x}_i - \b{\mu})(\b{x}_i - \b{\mu})^\top \\
&- 2 \b{\Lambda} \mathbb{E}_{\sim q^{(t)}(\b{z}_i)}\! [\b{z}_i] (\b{x}_i - \b{\mu})^\top + \b{\Lambda} \mathbb{E}_{\sim q^{(t)}(\b{z}_i)}\! [\b{z} \b{z}^\top] \b{\Lambda}^\top \Big] \bigg).
\end{aligned}
\end{equation}

\subsubsection{Summary of Factor Analysis Algorithm}

According to the derived Eqs. (\ref{equation_factor_analysis_z_given_x_mu_3}), (\ref{equation_factor_analysis_z_given_x_Sigma_3}), (\ref{equation_factor_analysis_Lambda}), and (\ref{equation_factor_analysis_Psi}), the EM algorithm in factor analysis is summarized as follows. The mean of data, $\b{\mu}$, is computed. Then, for every data point $\b{x}_i$, we iteratively solve as:
\begin{align*}
&\mathbb{E}_{\sim q^{(t)}(\b{z}_i)}[\b{z}_i] \gets \b{\Lambda}^{(t)\top} (\b{\Lambda}^{(t)}\b{\Lambda}^{(t)\top} + \b{\Psi}^{(t)})^{-1} (\b{x}_i - \b{\mu}), \\
&\mathbb{E}_{\sim q^{(t)}(\b{z}_i)}[\b{z}_i \b{z}_i^\top] \gets \b{I} - \b{\Lambda}^{(t)\top} (\b{\Lambda}^{(t)}\b{\Lambda}^{(t)\top} + \b{\Psi}^{(t)})^{-1} \b{\Lambda}^{(t)}, \\
&\b{\Lambda}^{(t+1)} \gets \Big(\sum_{i=1}^n (\b{\Psi}^{(t)})^{-1} (\b{x}_i - \b{\mu}) \mathbb{E}_{\sim q^{(t)}(\b{z}_i)}[\b{z}_i]^\top\Big) \nonumber \\
&~~~~~~~~~~~~~~~~~~~~~~~~~~ \Big( \sum_{i=1}^n (\b{\Psi}^{(t)})^{-1} \b{\Lambda}^{(t)} \mathbb{E}_{\sim q^{(t)}(\b{z}_i)}[\b{z}_i \b{z}_i^\top] \Big)^{-1}. \\
&\b{\Psi}^{(t+1)} \gets \frac{1}{n} \textbf{diag}\bigg( \sum_{i=1}^n \Big[ (\b{x}_i - \b{\mu})(\b{x}_i - \b{\mu})^\top \\
&~~~~~~~~~~~~~~~~~~~~~~~~~~ - 2 \b{\Lambda}^{(t+1)} \mathbb{E}_{\sim q^{(t)}(\b{z}_i)}\! [\b{z}_i] (\b{x}_i - \b{\mu})^\top \\
&~~~~~~~~~~~~~~~~~~~~~~~~~~ + \b{\Lambda}^{(t+1)} \mathbb{E}_{\sim q^{(t)}(\b{z}_i)}\! [\b{z} \b{z}^\top] \b{\Lambda}^{(t+1)} \Big] \bigg).
\end{align*}

Note that if data are centered as a pre-processing to factor analysis, i.e. $\b{\mu} = \b{0}$, the algorithm of factor analysis is simplified as:
\begin{align*}
&\mathbb{E}_{\sim q^{(t)}(\b{z}_i)}[\b{z}_i] \gets \b{\Lambda}^{(t)\top} (\b{\Lambda}^{(t)}\b{\Lambda}^{(t)\top} + \b{\Psi}^{(t)})^{-1} \b{x}_i, \\
&\mathbb{E}_{\sim q^{(t)}(\b{z}_i)}[\b{z}_i \b{z}_i^\top] \gets \b{I} - \b{\Lambda}^{(t)\top} (\b{\Lambda}^{(t)}\b{\Lambda}^{(t)\top} + \b{\Psi}^{(t)})^{-1} \b{\Lambda}^{(t)}, \\
&\b{\Lambda}^{(t+1)} \gets \Big(\sum_{i=1}^n (\b{\Psi}^{(t)})^{-1} \b{x}_i\, \mathbb{E}_{\sim q^{(t)}(\b{z}_i)}[\b{z}_i]^\top\Big) \nonumber \\
&~~~~~~~~~~~~~~~~~~~~~~~~~~ \Big( \sum_{i=1}^n (\b{\Psi}^{(t)})^{-1} \b{\Lambda}^{(t)} \mathbb{E}_{\sim q^{(t)}(\b{z}_i)}[\b{z}_i \b{z}_i^\top] \Big)^{-1}. \\
&\b{\Psi}^{(t+1)} \gets \frac{1}{n} \textbf{diag}\bigg( \sum_{i=1}^n \Big[ \b{x}_i \b{x}_i^\top - 2 \b{\Lambda}^{(t+1)} \mathbb{E}_{\sim q^{(t)}(\b{z}_i)}\! [\b{z}_i]\, \b{x}_i^\top \\
&~~~~~~~~~~~~~~~~~~~~~~~~~~ + \b{\Lambda}^{(t+1)} \mathbb{E}_{\sim q^{(t)}(\b{z}_i)}\! [\b{z} \b{z}^\top] \b{\Lambda}^{(t+1)} \Big] \bigg).
\end{align*}
As it can be seen, factor analysis does not have a closed-form solution and its solution, which are the projection matrix $\b{\Lambda}$ and the noise covariance matrix $\b{\Psi}$, are found iteratively until convergence. 

It is noteworthy that mixture of factor analysis \cite{ghahramani1996algorithm} also exists in the literature which considers a mixture distribution for the factor analysis and trains the parameters of mixture using EM \cite{ghojogh2019fitting}. 

\section{Probabilistic Principal Component Analysis}\label{section_PPCA}

\subsection{Main Idea of Probabilistic PCA}

Probabilistic PCA (PPCA) \cite{roweis1997algorithms,tipping1999probabilistic} is a special case of factor analysis where the variance of noise is equal in all dimensions of data space with covariance between dimensions, i.e.: 
\begin{align}\label{equation_PPCA_Psi}
\b{\Psi} = \sigma^2 \b{I}.
\end{align}
In other words, PPCA considers an isotropic noise in its formulation.
Therefore, Eq. (\ref{equation_factor_analysis_prior_epsilon}) is simplified to:
\begin{align}
&\mathbb{P}(\b{\epsilon}) = \mathcal{N}(\b{0}, \sigma^2 \b{I}). \label{equation_PPCA_prior_epsilon}
\end{align}
Because of having zero covariance of noise between different dimensions, PPCA assumes that the data points are independent of each other given latent variables. 
As depicted in Figure \ref{figure_PPCA}, PPCA can be illustrated as a graphical model, where the visible data variable is conditioned on the latent variable and the isotropic noise random variable.

\begin{figure}[!t]
\centering
\includegraphics[width=1.1in]{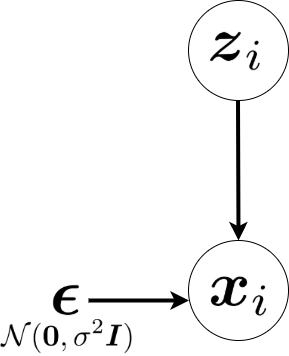}
\caption{The graphical model for PPCA.}
\label{figure_PPCA}
\end{figure}

\subsection{MLE for Probabilistic PCA}

As PPCA is a special case of factor analysis, it also is solved using EM. 
Similar to factor analysis, it can be solved iteratively using EM \cite{roweis1997algorithms}. However, one can also find a closed-form solution to its EM approach \cite{tipping1999probabilistic}. Hence, by restricting the noise covariance to be isotropic, its solution becomes simpler and closed-form. The iterative approach is as we had in factor analysis. Here, we derive the closed-form solution.

Consider the likelihood or the marginal distribution of data points $\{\b{x}_i \in \mathbb{R}^d \}_{i=1}^n$ which is Eq. (\ref{equation_factor_analysis_prior_of_data}). The log-likelihood of data is:
\begin{align*}
&\sum_{i=1}^n \log \mathbb{P}(\b{x}_i) \overset{(\ref{equation_factor_analysis_prior_of_data})}{=} \sum_{i=1}^n \log \mathcal{N}(\b{\mu}, \b{\Lambda}\b{\Lambda}^\top + \sigma^2 \b{I}) \\
&= \sum_{i=1}^n \bigg[ \log\big( \frac{1}{(2\pi)^{d/2} |\b{\Lambda}\b{\Lambda}^\top + \sigma^2 \b{I}|^{1/2}} \times \\
&~~~~~\exp\Big(\!\!- \frac{(\b{x}_i - \b{\mu})^\top (\b{\Lambda}\b{\Lambda}^\top + \sigma^2 \b{I})^{-1} (\b{x}_i - \b{\mu})}{2}\Big) \big) \bigg]
\end{align*}
\begin{align}
&= \underbrace{-\frac{d\,n}{2} \log(2\pi)}_\text{constant} - \frac{n}{2} \log |\b{\Lambda}\b{\Lambda}^\top + \sigma^2 \b{I}| \nonumber \\
&~~~~ - \sum_{i=1}^n \frac{1}{2} (\b{x}_i - \b{\mu})^\top (\b{\Lambda}\b{\Lambda}^\top + \sigma^2 \b{I})^{-1} (\b{x}_i - \b{\mu}) \nonumber \\
&= \underbrace{-\frac{d\,n}{2} \log(2\pi)}_\text{constant} - \frac{n}{2} \log |\b{\Lambda}\b{\Lambda}^\top + \sigma^2 \b{I}| \nonumber \\
&~~~~~~~~~~~~ - \frac{n}{2} \textbf{tr}\big((\b{\Lambda}\b{\Lambda}^\top + \sigma^2 \b{I})^{-1} \b{S}_x\big), \nonumber
\end{align}
where $\b{S}_x \in \mathbb{R}^{d \times d}$ is the sample covariance matrix of data:
\begin{align}
&\b{S}_x := \frac{1}{n} \sum_{i=1}^n (\b{x}_i - \b{\mu})  (\b{x}_i - \b{\mu})^\top.
\end{align}
We use MLE where the variables of maximization optimization are the projection matrix $\b{\Lambda}$ and the noise variance $\sigma$:
\begin{align}
&\max_{\b{\Lambda}, \sigma}~ \Big( \underbrace{-\frac{d\,n}{2} \log(2\pi)}_\text{constant} - \frac{n}{2} \log |\b{\Lambda}\b{\Lambda}^\top + \sigma^2 \b{I}| \nonumber \\
&~~~~~~~~~~~~~~~~~~~~~ - \frac{n}{2} \textbf{tr}\big((\b{\Lambda}\b{\Lambda}^\top + \sigma^2 \b{I})^{-1} \b{S}_x\big) \Big). \label{equation_PPCA_log_likelihood}
\end{align}
It is noteworthy that literature usually defines: 
\begin{align}
&\mathbb{R}^{d \times d} \ni \b{C} := (\b{\Lambda}\b{\Lambda}^\top + \sigma^2 \b{I}_{d \times d}). \\
&\mathbb{R}^{p \times p} \ni \b{M} := (\b{\Lambda}^\top \b{\Lambda} + \sigma^2 \b{I}_{p \times p}).
\end{align}
According to the matrix inversion lemma, we have:
\begin{align}
\b{C}^{-1} = \sigma^{-1} \b{I}_{d \times d} - \sigma^{-2} \b{\Lambda} \b{M}^{-1} \b{\Lambda}^\top.
\end{align}
This inversion is interesting because the inverse of a $(d \times d)$ matrix $\b{C}$ is reduced to inversion of a $(p \times p)$ matrix $\b{M}$ which is much simpler because we usually have $p \ll d$.

\subsubsection{MLE for Determining $\b{\Lambda}$}

Taking the derivative of Eq. (\ref{equation_PPCA_log_likelihood}) w.r.t. $\b{\Lambda} \in \mathbb{R}^{d \times p}$ and setting it to zero is:
\begin{align*}
&\frac{\partial\, \text{Eq.} (\ref{equation_PPCA_log_likelihood})}{\partial \b{\Lambda}} = -n \big( (\b{\Lambda}\b{\Lambda}^\top + \sigma^2 \b{I})^{-1} \b{S}_x (\b{\Lambda}\b{\Lambda}^\top + \sigma^2 \b{I})^{-1} \b{\Lambda} \\
&~~~~~~~~~~~~~~~~~~~~~ - (\b{\Lambda}\b{\Lambda}^\top + \sigma^2 \b{I})^{-1} \b{\Lambda} \big) \overset{\text{set}}{=} \b{0} \\
& \implies (\b{\Lambda}\b{\Lambda}^\top + \sigma^2 \b{I})^{-1} \b{S}_x (\b{\Lambda}\b{\Lambda}^\top + \sigma^2 \b{I})^{-1} \b{\Lambda} \\
&~~~~~~~~~~~~~~~~~~~~~~~~~~~~~~~~~~~~~~~~~~ = (\b{\Lambda}\b{\Lambda}^\top + \sigma^2 \b{I})^{-1} \b{\Lambda} \\
&\implies \b{S}_x (\b{\Lambda}\b{\Lambda}^\top + \sigma^2 \b{I})^{-1} \b{\Lambda} = \b{\Lambda},
\end{align*}
whose trivial solutions are $\b{\Lambda} = \b{0}$ and $\b{S}_x = (\b{\Lambda}\b{\Lambda}^\top + \sigma^2 \b{I})$ which are not valid. 
For the non-trivial solution, consider the Singular Value Decomposition (SVD) $\mathbb{R}^{d \times p} \ni \b{\Lambda} = \b{U} \b{L} \b{V}^\top$ where $\b{U} \in \mathbb{R}^{d \times p}$ and $\b{V} \in \mathbb{R}^{p \times p}$ contain the left and right singular vectors, respectively, and $\b{L} \in \mathbb{R}^{p \times p}$ is the diagonal matrix containing the singular values denoted by $\{l_j\}_{j=1}^p$. 
Moreover, note that $\textbf{tr}(\b{\Lambda}\b{\Lambda}^\top) = \textbf{tr}(\b{U} \b{L} \b{V}^\top \b{V} \b{L} \b{U}^\top) = \textbf{tr}(\b{U} \b{L} \b{L} \b{U}^\top) = \textbf{tr}(\b{U}^\top \b{U} \b{L}^2) = \textbf{tr}(\b{L}^2)$ because $\b{U}$ and $\b{V}$ are orthogonal matrices. 

From the previous calculations, we have \cite{hauskrecht2007cs3750}:
\begin{align}
&\b{S}_x \b{U} \b{L} (\b{L}^2 + \sigma^2 \b{I})^{-1} \b{V}^\top = \b{U} \b{L} \b{V}^\top \nonumber \\
&\implies \b{S}_x \b{U} \b{L} (\b{L}^2 + \sigma^2 \b{I})^{-1} = \b{U} \b{L} \nonumber \\
&\implies \b{S}_x \b{U} \b{L} = \b{U} (\b{L}^2 + \sigma^2 \b{I}) \b{L} \nonumber \\
&\implies \b{S}_x \b{U} = \b{U} (\b{L}^2 + \sigma^2 \b{I}), \label{equation_PPCA_eigen_problem}
\end{align}
which is an eigenvalue problem \cite{ghojogh2019eigenvalue} for the covariance matrix $\b{S}_x$ where the columns of $\b{U}$ are the eigenvectors of $\b{S}_x$ and the eigenvalues are $\sigma^2 + l_j^2$. Recall that $\sigma$ is the variance of noise in different dimensions and $l_j$ is the $j$-th singular value of $\b{\Lambda}$ (sorted from largest to smallest). We denote the $j$-th eigenvalue of the covariance matrix $\b{S}_x$ by:
\begin{align}
\delta_j := \sigma^2 + l_j^2 \implies l_j = (\delta_j - \sigma^2)^{(1/2)}.
\end{align}
We consider only the top $p$ singular values $l_j$ and the top $p$ eigenvalues $\delta_j$; substituting the singular values in the SVD of the projection matrix $\b{\Lambda}$ results in:
\begin{align*}
\b{\Lambda} = \b{U} \b{L} \b{V}^\top = \b{U} (\b{\Delta}_p - \sigma^2 \b{I})^{(1/2)} \b{V}^\top,
\end{align*}
where $\b{\Delta}_p := \textbf{diag}(\delta_1, \dots, \delta_p)$.
However, as Eq. (\ref{equation_PPCA_eigen_problem}) does not include any $\b{V}$, we can replace $\b{V}^\top$ with any arbitrary orthogonal matrix $\b{R} \in \mathbb{R}^{p \times p}$ \cite{tipping1999probabilistic}:
\begin{align}\label{equation_PPCA_Lambda}
\b{\Lambda} = \b{U} (\b{\Delta}_p - \sigma^2 \b{I})^{(1/2)} \b{R}.
\end{align}
The arbitrary orthogonal matrix $\b{R}$ is a rotation matrix which rotates data in projection. It is arbitrary because rotation is not important in dimensionality reduction (the relative distances of embedded data points do not change if all embedded points are rotated). A simple choice for this rotation matrix is $\b{R} = \b{I}$ which results in:
\begin{align}
\b{\Lambda} = \b{U} (\b{\Delta}_p - \sigma^2 \b{I})^{(1/2)}.
\end{align}

\subsubsection{MLE for Determining $\sigma$}

If we substitute Eq. (\ref{equation_PPCA_Lambda}) in the log-likelihood, Eq. (\ref{equation_PPCA_log_likelihood}), the log-likelihood becomes \cite{hauskrecht2007cs3750}:
\begin{align}
&\max_{\sigma}~ -\frac{n}{2} \Big( \underbrace{d \log(2\pi)}_\text{constant} + \sum_{j=1}^p \log (\delta_j) \nonumber \\
&~~~~~~~~~ + \sum_{j=p+1}^d \delta_j + (d-p) \log ((\sigma^2)^{-1}) + p \Big). \label{equation_PPCA_log_likelihood_for_sigma}
\end{align}
Note that we have $d$ eigenvalues $\{\delta_j\}_{j=1}^d$ because the covariance matrix $\b{S}_x$ is a $(d \times d)$ matrix. However, as we have only $p$ singular values $\{l_j\}_{j=1}^p$, the eigenvalues $\{\delta_j\}_{j=p+1}^d$ are very small. 

Taking the derivative of Eq. (\ref{equation_PPCA_log_likelihood_for_sigma}) w.r.t. $\sigma^2$ and setting it to zero is \cite{tipping1999probabilistic}:
\begin{align}
&\frac{\partial\, \text{Eq.} (\ref{equation_PPCA_log_likelihood_for_sigma})}{\partial \sigma^2} = -\frac{n}{2} \big( 0 + \sum_{j=p+1}^d \delta_j + (d-p)\, \sigma^2 + 0 \big) \overset{\text{set}}{=} 0 \nonumber \\
&\implies \sigma^2 = \frac{1}{d-p} \sum_{j=p+1}^d \delta_j.
\end{align}

\subsubsection{Summary of MLE Formulas}

In summary, the MLE estimations for the variables of PPCA are:
\begin{align}
&\sigma^2 = \frac{1}{d-p} \sum_{j=p+1}^d \delta_j, \label{equation_PPCA_sigma_squared} \\
&\b{\Lambda} = \b{U} (\b{\Delta}_p - \sigma^2 \b{I})^{(1/2)} \b{R} = \b{U} (\b{\Delta}_p - \sigma^2 \b{I})^{(1/2)}. \label{equation_PPCA_Lambda_without_R}
\end{align}
Note that Eq. (\ref{equation_PPCA_sigma_squared}) is a measure of the variance lost in the projection by the projection matrix $\b{\Lambda}$. The Eq. (\ref{equation_PPCA_Lambda_without_R}) is the projection or mapping matrix from the latent space to the data space. 

\subsection{Zero Noise Limit: PCA Is a Special Case of Probabilistic PCA}

Recall the posterior which is Eq. (\ref{equation_factor_analysis_z_given_x}). According to Eqs. (\ref{equation_factor_analysis_z_given_x_mu_2}), (\ref{equation_factor_analysis_z_given_x_Sigma_2}), and (\ref{equation_PPCA_Psi}), the posterior in PPCA is:
\begin{equation}\label{equation_PPCA_posterior}
\begin{aligned}
q(\b{z}_i) = \mathbb{P}(\b{z}_i\, |\, \b{x}_i) = \mathcal{N}\Big(&\b{\Lambda}^\top (\b{\Lambda}\b{\Lambda}^\top + \sigma^2 \b{I})^{-1} (\b{x}_i - \b{\mu}), \\
&\b{I} - \b{\Lambda}^\top (\b{\Lambda}\b{\Lambda}^\top + \sigma^2 \b{I})^{-1} \b{\Lambda} \Big).
\end{aligned}
\end{equation}

Consider zero noise limit where the variance of noise goes to zero:
\begin{align}
\lim_{\sigma^2 \rightarrow 0} \mathbb{P}(\b{\epsilon}) = \lim_{\sigma^2 \rightarrow 0} \mathcal{N}(\b{0}, \sigma^2 \b{I}) = \mathcal{N}(\b{0}, \lim_{\sigma^2 \rightarrow 0} (\sigma^2) \b{I}).
\end{align}
In this case, the uncertainty of PPCA almost disappears. In the following, we show that in zero noise limit, PPCA is reduced to PCA \cite{ghojogh2019unsupervised,jolliffe2016principal} and this explains why the PPCA method is a probabilistic approach to PCA. 

In the zero noise limit, the posterior becomes:
\begin{equation}
\begin{aligned}
&\lim_{\sigma^2 \rightarrow 0} q(\b{z}_i) = \lim_{\sigma^2 \rightarrow 0} \mathbb{P}(\b{z}_i\, |\, \b{x}_i) \\
&\overset{(\ref{equation_PPCA_posterior})}{=} \mathcal{N}\Big(\b{\Lambda}^\top (\b{\Lambda}\b{\Lambda}^\top)^{-1} (\b{x}_i - \b{\mu}), \b{I} - \b{\Lambda}^\top (\b{\Lambda}\b{\Lambda}^\top)^{-1} \b{\Lambda} \Big) \\
&\overset{(a)}{=} \mathcal{N}\Big((\b{\Lambda}^\top \b{\Lambda})^{-1} \b{\Lambda}^\top (\b{x}_i - \b{\mu}), \b{I} - (\b{\Lambda}^\top \b{\Lambda})^{-1} \b{\Lambda}^\top \b{\Lambda} \Big). \label{equation_PPCA_zero_noise_limit_posterior}
\end{aligned}
\end{equation}
where $(a)$ is because according to {\citep[footnote 4]{roweis1997algorithms}}, we have:
\begin{align}
\b{\Lambda}^\top (\b{\Lambda}\b{\Lambda}^\top)^{-1} = (\b{\Lambda}^\top \b{\Lambda})^{-1} \b{\Lambda}^\top.
\end{align}

On the other hand, according to {\citep[Appendix C]{tipping1999probabilistic}}, PCA minimizes the reconstruction error as:
\begin{equation}
\begin{aligned}
& \underset{\b{\Lambda}}{\text{minimize}}
& & ||(\b{X} - \b{\mu}) - \b{\Lambda}\b{\Lambda}^\top(\b{X} - \b{\mu})||_F^2,
\end{aligned}
\end{equation}
where $\|.\|_F$ denotes the Frobenius norm of matrix. See \cite{ghojogh2019unsupervised} for more details on minimization of reconstruction error by PCA. 

Instead of minimizing the reconstruction error, one may minimize the reconstruction error for the mean of posterior {\citep[Appendix C]{tipping1999probabilistic}}:
\begin{equation}
\begin{aligned}
& \underset{\b{\Lambda}}{\text{minimize}}
& & ||(\b{X} - \b{\mu}) - \b{\Lambda} (\b{\Lambda}^\top \b{\Lambda})^{-1} \b{\Lambda}^\top (\b{X} - \b{\mu})||_F^2.
\end{aligned}
\end{equation}
This is the minimization of reconstruction error after projection by $(\b{\Lambda}^\top \b{\Lambda})^{-1} \b{\Lambda}^\top$. 
According to the posterior in the zero noise limit (see Eq. (\ref{equation_PPCA_zero_noise_limit_posterior})), this is equivalent to PPCA in the zero noise limit model. Hence, PCA is a deterministic special case of PPCA where the variance of noise goes to zero.

\subsection{Other Variants of Probabilistic PCA}

There exist some other variants of PPCA. We briefly mention them here. 
PPCA has a hyperparameter which is the dimensionality of latent space $p$. 
Bayesian PCA \cite{bishop1999bayesian} models this hyperparameter as another latent random variable which is learned during the EM training. 

According to Eqs. (\ref{equation_x_z_epsilon}), (\ref{equation_factor_analysis_prior_z}), and (\ref{equation_PPCA_prior_epsilon}), note that PPCA uses Gaussian distributions. PPCA with Student-t distribution \cite{zhao2006probabilistic} has been proposed which uses t distributions. This change may improve the embedding of PPCA because of the heavier tails of Student-t distribution compared to Gaussian distribution. This avoids the crowding problem which has motivated the proposal of t-SNE \cite{ghojogh2020stochastic}. 

Sparse PPCA  \cite{guan2009sparse,mattei2016globally} has inserted sparsity into PPCA.
Supervised PPCA \cite{yu2006supervised} makes use of class labels in the formulation of PPCA. 
Mixture of PPCA \cite{tipping1999mixtures} uses mixture of distributions in the formulation PPCA. It trains the parameters of a mixture distribution using EM \cite{ghojogh2019fitting}. 
Generalized PPCA for correlated data is another recent variant of PPCA \cite{gu2020generalized}.

\section{Variational Autoencoder}\label{section_VAE}

Variational Autoencoder (VAE) \cite{kingma2014auto} applies variational inference, i.e., maximizes the ELBO, but in an autoencoder setup and makes it differentiable for the backpropagation training \cite{rumelhart1986learning}. 
As Fig. \ref{figure_VAE} shows, VAE includes an encoder and a decoder, each of which can have several network layers.
A latent space is learned between the encoder and decoder. The latent variable $\b{z}_i$ is sampled from the latent space. In the following, we explain the encoder and decoder parts. The input of encoder in VAE is the data point $\b{x}_i$ and the output of decoder in VAE is its reconstruction $\b{x}_i$.

\begin{figure}[!t]
\centering
\includegraphics[width=2.7in]{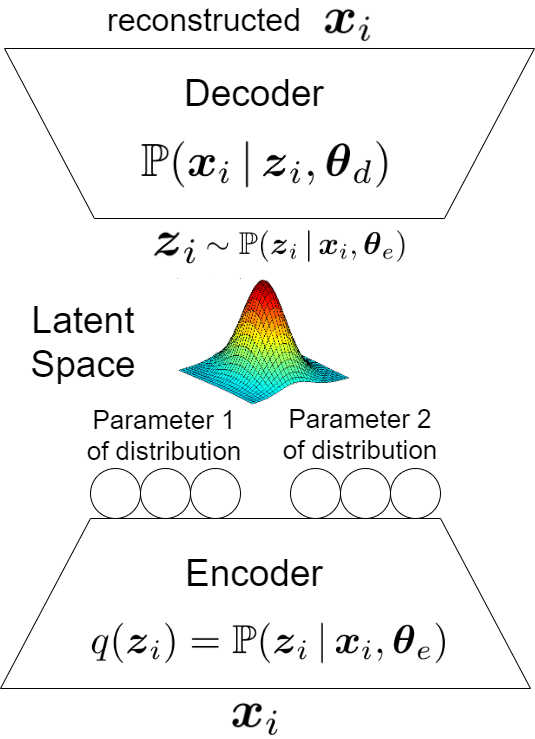}
\caption{The structure of a variational autoencoder.}
\label{figure_VAE}
\end{figure}

\subsection{Parts of Variational Autoencoder}

\subsubsection{Encoder of Variational Autoencoder}

The encoder of VAE models the distribution $q(\b{z}_i) = \mathbb{P}(\b{z}_i\, |\, \b{x}_i, \b{\theta}_e)$ where the parameters of distribution $\b{\theta}_e$ are the weights of encoder layers in VAE.
The input and output of encoder are $\b{x}_i \in \mathbb{R}^d$ and $\b{z}_i \in \mathbb{R}^p$, respectively. 
As Fig. \ref{figure_VAE} depicts, the output neurons of encoder are supposed to determine the parameters of the conditional distribution $\mathbb{P}(\b{z}_i\, |\, \b{x}_i, \b{\theta}_e)$. If this conditional distribution has $m$ number of parameters, we have $m$ sets of output neurons from the encoder, denoted by $\{\b{e}_j\}_{j=1}^m$. The dimensionality of these sets may differ depending on the size of the parameters. 

For example, let the latent space be $p$-dimensional, i.e., $\b{z}_i \in \mathbb{R}^p$. If the distribution $\mathbb{P}(\b{z}_i\, |\, \b{x}_i, \b{\theta}_e)$ is a multivariate Gaussian distribution, we have two sets of output neurons for encoder where one set has $p$ neurons for the mean of this distribution $\b{\mu}_{z|x} = \b{e}_1 \in \mathbb{R}^d$ and the other set has $(p \times p)$ neurons for the covariance of this distribution $\b{\Sigma}_{z|x} = \text{matrix form of } \b{e}_2 \in \mathbb{R}^{p \times p}$. If the covariance matrix is diagonal, the second set has $p$ neurons rather than $(p \times p)$ neurons. In this case, we have $\b{\Sigma}_{z|x} = \textbf{diag}(\b{e}_2) \in \mathbb{R}^{d \times d}$.
Any distribution with any number of parameters can be chosen for $\mathbb{P}(\b{z}_i\, |\, \b{x}_i, \b{\theta}_e)$ but the multivariate Gaussian with diagonal covariance is very well-used:
\begin{align}
q(\b{z}_i) = \mathbb{P}(\b{z}_i\, |\, \b{x}_i, \b{\theta}_e) = \mathcal{N}(\b{z}_i\, |\, \b{\mu}_{z|x}, \b{\Sigma}_{z|x}).
\end{align}

Let the network weights for the output sets of encoder, $\{\b{e}_j\}_{j=1}^m$, be denoted by $\{\b{\theta}_{e,j}\}_{j=1}^m$. As the input of encoder is $\b{x}_i$, the $j$-th output set of encoder can be written as $\b{e}_j(\b{x}_i, \b{\theta}_{e,j})$. In the case of multivariate Gaussian distribution for the latent space, the parameters are $\b{\mu}_{z|x} = \b{e}_1(\b{x}_i, \b{\theta}_{e,1})$ and $\b{\Sigma}_{z|x} = \textbf{diag}(\b{e}_2(\b{x}_i, \b{\theta}_{e,2}))$. 

\subsubsection{Sampling the Latent Variable}

When the data point $\b{x}_i$ is fed as input to the encoder, the parameters of the conditional distribution $q(\b{z}_i)$ are obtained; hence, the distribution of latent space, which is $q(\b{z}_i)$, is determined corresponding to the data point $\b{x}_i$. Now, in the latent space, we sample the corresponding latent variable from the distribution of latent space:
\begin{align}\label{equation_VAE_sampling}
\b{z}_i \sim q(\b{z}_i) = \mathbb{P}(\b{z}_i\, |\, \b{x}_i, \b{\theta}_e).
\end{align}
This latent variable is fed as input to the decoder which is explained in the following. 

\subsubsection{Decoder of Variational Autoencoder}

As Fig. \ref{figure_VAE} shows, the decoder of VAE models the conditional distribution $\mathbb{P}(\b{x}_i\, |\, \b{z}_i, \b{\theta}_d)$ where $\b{\theta}_d$ are the weights of decoder layers in VAE. 
The input and output of decoder are $\b{z}_i \in \mathbb{R}^p$ and $\b{x}_i \in \mathbb{R}^d$, respectively. 
The output neurons of decoder are supposed to either generate the reconstructed data point or determine the parameters of the conditional distribution $\mathbb{P}(\b{x}_i\, |\, \b{z}_i, \b{\theta}_d)$; the former is more common. In the latter case, if this conditional distribution has $l$ number of parameters, we have $l$ sets of output neurons from the decoder, denoted by $\{\b{d}_j\}_{j=1}^l$. The dimensionality of these sets may differ depending the size of every parameters. 
The example of multivariate Gaussian distribution also can be mentioned for the decoder. 
Let the network weights for the output sets of decoder, $\{\b{d}_j\}_{j=1}^l$, be denoted by $\{\b{\theta}_{d,j}\}_{j=1}^l$. As the input of decoder is $\b{z}_i$, the $j$-th output set of decoder can be written as $\b{d}_j(\b{z}_i, \b{\theta}_{d,j})$. 

\subsection{Training Variational Autoencoder with Expectation Maximization}

% \subsubsection{Expectation Maximization for Training}

We use EM for training the VAE. 
Recall Eqs. (\ref{equation_variational_inference_E_step}) and (\ref{equation_variational_inference_M_step}) for EM in variational inference. 
Inspired by that, VAE uses EM for training where the ELBO is a function of encoder weights $\b{\theta}_e$, decoder weights $\b{\theta}_d$, and data point $\b{x}_i$:
\begin{align}
&\text{E-step:} ~~~~~ \b{\theta}_e^{(t)} := \arg\max_q~~~~ \mathcal{L}(\b{\theta}_e, \b{\theta}_d^{(t-1)}, \b{x}_i), \label{equation_VAE_E_step} \\
&\text{M-step:} ~~~~~ \b{\theta}_d^{(t)} := \arg\max_q~~~~ \mathcal{L}(\b{\theta}_e^{(t)}, \b{\theta}_d, \b{x}_i). \label{equation_VAE_M_step}
\end{align}
We can simplify this iterative optimization algorithm by alternating optimization \cite{jain2017non} where we take a step of gradient ascent optimization in every iteration. We consider mini-batch stochastic gradient ascent and take training data in batches where $b$ denotes the mini-batch size. Hence, the optimization is:
\begin{align}
&\text{E-step:} ~~~~~ \b{\theta}_e^{(t)} := \b{\theta}_e^{(t-1)} + \eta_e \frac{\partial \sum_{i=1}^b \mathcal{L}(\b{\theta}_e, \b{\theta}_d^{(t-1)}, \b{x}_i)}{\partial \b{\theta}_e}, \label{equation_VAE_E_step_gradient} \\
&\text{M-step:} ~~~~~ \b{\theta}_d^{(t)} := \b{\theta}_d^{(t-1)} + \eta_d \frac{\partial \sum_{i=1}^b \mathcal{L}(\b{\theta}_e^{(t)}, \b{\theta}_d, \b{x}_i)}{\partial \b{\theta}_d}, \label{equation_VAE_M_step_gradient}
\end{align}
where $\eta_e$ and $\eta_d$ are the learning rates for $\b{\theta}_e$ and $\b{\theta}_d$, respectively. 

The ELBO is simplified as:
\begin{align}
&\sum_{i=1}^b \mathcal{L}(q, \b{\theta}) \overset{(\ref{ELBO_equation1})}{=} - \sum_{i=1}^b \text{KL}\big(q(\b{z}_i)\, \|\, \mathbb{P}(\b{x}_i, \b{z}_i\, |\, \b{\theta}_d)\big) \nonumber \\
&~~~~~~~ \overset{(\ref{equation_E_step_variationalInference})}{=} - \sum_{i=1}^b \text{KL}\big(\mathbb{P}(\b{z}_i\, |\, \b{x}_i, \b{\theta}_e)\, \|\, \mathbb{P}(\b{x}_i, \b{z}_i\, |\, \b{\theta}_d)\big). \label{equation_VAE_ELBO_KL_divergence}
\end{align}
Note that the parameter of $\mathbb{P}(\b{x}_i, \b{z}_i\, |\, \b{\theta}_d)$ is $\b{\theta}_d$ because $\b{z}_i$ is generated after the encoder and before the decoder. 

There are different ways for approximating the KL divergence in Eq. (\ref{equation_VAE_ELBO_KL_divergence}) \cite{hershey2007approximating,duchi2007derivations}.
We can simplify the ELBO in at least two different ways which are explained in the following.

% \hfill\break
% \textbf{Simplification Type 1:}

\subsubsection{Simplification Type 1}

We continue the simplification of ELBO:
\begin{align}
&\sum_{i=1}^b \mathcal{L}(q, \b{\theta}) = - \sum_{i=1}^b \text{KL}\big(\mathbb{P}(\b{z}_i\, |\, \b{x}_i, \b{\theta}_e)\, \|\, \mathbb{P}(\b{x}_i, \b{z}_i\, |\, \b{\theta}_d)\big) \nonumber \\
&= - \sum_{i=1}^b \mathbb{E}_{\sim q^{(t-1)}(\b{z}_i)} \Big[\log \big(\frac{\mathbb{P}(\b{z}_i\, |\, \b{x}_i, \b{\theta}_e)}{\mathbb{P}(\b{x}_i, \b{z}_i\, |\,  \b{\theta}_d)}\big)\Big] \nonumber \\
&= - \sum_{i=1}^b \mathbb{E}_{\sim \mathbb{P}(\b{z}_i\, |\, \b{x}_i, \b{\theta}_e)} \Big[\log \big(\frac{\mathbb{P}(\b{z}_i\, |\, \b{x}_i, \b{\theta}_e)}{\mathbb{P}(\b{x}_i, \b{z}_i\, |\,  \b{\theta}_d)}\big)\Big]. \label{equation_VAE_simplification1_ELBO_withExpectation}
\end{align}
This expectation can be approximated using Monte Carlo approximation \cite{ghojogh2020sampling} where we draw $\ell$ samples $\{\b{z}_{i,j}\}_{j=1}^\ell$, corresponding to the $i$-th data point, from the conditional distribution distribution as:
\begin{align}
\b{z}_{i,j} \sim \mathbb{P}(\b{z}_i\, |\, \b{x}_i, \b{\theta}_e), \quad \forall j \in \{1,\dots,\ell\}.
\end{align}
Monte Carlo approximation \cite{ghojogh2020sampling}, in general, approximates expectation as:
\begin{align}
\mathbb{E}_{\sim \mathbb{P}(\b{z}_i\, |\, \b{x}_i, \b{\theta}_e)} \big[ f(\b{z}_i) \big] \approx \frac{1}{\ell} \sum_{j=1}^\ell f(\b{z}_{i,j}),
\end{align}
where $f(\b{z}_i)$ is a function of $\b{z}_i$.
Here, the approximation is:
\begin{align}
&\sum_{i=1}^b \mathcal{L}(q, \b{\theta}) \approx \sum_{i=1}^b \widetilde{\mathcal{L}}(q, \b{\theta}) \nonumber \\
&= - \sum_{i=1}^b \frac{1}{\ell} \sum_{j=1}^\ell \log \big(\frac{\mathbb{P}(\b{z}_{i,j}\, |\, \b{x}_i, \b{\theta}_e)}{\mathbb{P}(\b{x}_i, \b{z}_{i,j}\, |\,  \b{\theta}_d)}\big) \nonumber \\
&= \sum_{i=1}^b \frac{1}{\ell} \sum_{j=1}^\ell \! \Big[ \log \big( \mathbb{P}(\b{x}_i, \b{z}_{i,j}\, |\,  \b{\theta}_d) \big) \nonumber \\
&~~~~~~~~~~~~~~~~~~~~~~~~~~~~~~~~~~~ - \log \big( \mathbb{P}(\b{z}_{i,j}\, |\, \b{x}_i, \b{\theta}_e) \big) \Big].
\end{align}

% \hfill\break
% \textbf{Simplification Type 2:}

\subsubsection{Simplification Type 2}

We can simplify the ELBO using another approach:
\begin{align}
&\sum_{i=1}^b \mathcal{L}(q, \b{\theta}) = - \sum_{i=1}^b \text{KL}\big(\mathbb{P}(\b{z}_i\, |\, \b{x}_i, \b{\theta}_e)\, \|\, \mathbb{P}(\b{x}_i, \b{z}_i\, |\, \b{\theta}_d)\big) \nonumber \\
&= - \sum_{i=1}^b \int \mathbb{P}(\b{z}_i\, |\, \b{x}_i, \b{\theta}_e) \log \big(\frac{\mathbb{P}(\b{z}_i\, |\, \b{x}_i, \b{\theta}_e)}{\mathbb{P}(\b{x}_i, \b{z}_i\, |\,  \b{\theta}_d)}\big)\, d\b{z}_i \nonumber \\
&= - \sum_{i=1}^b \int \mathbb{P}(\b{z}_i\, |\, \b{x}_i, \b{\theta}_e) \log \big(\frac{\mathbb{P}(\b{z}_i\, |\, \b{x}_i, \b{\theta}_e)}{\mathbb{P}(\b{x}_i\, |\, \b{z}_i, \b{\theta}_d)\, \mathbb{P}(\b{z}_i)}\big)\, d\b{z}_i \nonumber 
\end{align}
\begin{align}
&= - \sum_{i=1}^b \int \mathbb{P}(\b{z}_i\, |\, \b{x}_i, \b{\theta}_e) \log \big(\frac{\mathbb{P}(\b{z}_i\, |\, \b{x}_i, \b{\theta}_e)}{\mathbb{P}(\b{z}_i)}\big)\, d\b{z}_i \nonumber \\
&~~~~~ + \sum_{i=1}^b \int \mathbb{P}(\b{z}_i\, |\, \b{x}_i, \b{\theta}_e) \log\big(\mathbb{P}(\b{x}_i\, |\, \b{z}_i, \b{\theta}_d)\big)\, d\b{z}_i \nonumber \\
&= - \sum_{i=1}^b \text{KL}\big( \mathbb{P}(\b{z}_i\, |\, \b{x}_i, \b{\theta}_e)\, \|\, \mathbb{P}(\b{z}_i) \big) \nonumber \\
&~~~~~ + \sum_{i=1}^b \mathbb{E}_{\sim \mathbb{P}(\b{z}_i\, |\, \b{x}_i, \b{\theta}_e)} \Big[ \log\big(\mathbb{P}(\b{x}_i\, |\, \b{z}_i, \b{\theta}_d)\big) \Big]. \label{equation_VAE_simplification2_ELBO_withExpectation}
\end{align}
The second term in the above equation can be estimated using Monte Carlo approximation \cite{ghojogh2020sampling} where we draw $\ell$ samples $\{\b{z}_{i,j}\}_{j=1}^\ell$ from $\mathbb{P}(\b{z}_i\, |\, \b{x}_i, \b{\theta}_e)$:
\begin{align}
&\sum_{i=1}^b \mathcal{L}(q, \b{\theta}) \approx \sum_{i=1}^b \widetilde{\mathcal{L}}(q, \b{\theta}) \nonumber \\
&~~~~~~~~~~~~~ = - \sum_{i=1}^b \text{KL}\big( \mathbb{P}(\b{z}_i\, |\, \b{x}_i, \b{\theta}_e)\, \|\, \mathbb{P}(\b{z}_i) \big) \nonumber \\
&~~~~~~~~~~~~~~~~~ + \sum_{i=1}^b \frac{1}{\ell} \sum_{j=1}^{\ell} \log\big(\mathbb{P}(\b{x}_i\, |\, \b{z}_{i,j}, \b{\theta}_d)\big). \label{equation_VAE_simplification2}
\end{align}

The first term in the above equation can be converted to expectation and then computed using Monte Monte Carlo approximation \cite{ghojogh2020sampling} again, where we draw $\ell$ samples $\{\b{z}_{i,j}\}_{j=1}^\ell$ from $\mathbb{P}(\b{z}_i\, |\, \b{x}_i, \b{\theta}_e)$:
\begin{align}
&\sum_{i=1}^b \mathcal{L}(q, \b{\theta}) \approx \sum_{i=1}^b \widetilde{\mathcal{L}}(q, \b{\theta}) \nonumber \\
&~~~ = - \sum_{i=1}^b \mathbb{E}_{\sim \mathbb{P}(\b{z}_i\, |\, \b{x}_i, \b{\theta}_e)} \Big[ \log \big(\frac{\mathbb{P}(\b{z}_i\, |\, \b{x}_i, \b{\theta}_e)}{\mathbb{P}(\b{z}_i)} \big) \Big] \nonumber \\
&~~~~~~~ + \sum_{i=1}^b \frac{1}{\ell} \sum_{j=1}^{\ell} \log\big(\mathbb{P}(\b{x}_i\, |\, \b{z}_{i,j}, \b{\theta}_d)\big) \nonumber \\
&~~~ \approx - \sum_{i=1}^b \frac{1}{\ell} \sum_{j=1}^{\ell} \log \big( \mathbb{P}(\b{z}_{i,j}\, |\, \b{x}_i, \b{\theta}_e) \big) - \log \big( \mathbb{P}(\b{z}_{i,j}) \big) \nonumber \\
&~~~~~~~ + \sum_{i=1}^b \frac{1}{\ell} \sum_{j=1}^{\ell} \log\big(\mathbb{P}(\b{x}_i\, |\, \b{z}_{i,j}, \b{\theta}_d)\big).
\end{align}

In case we have some families of distributions, such as Gaussian distributions, for $\mathbb{P}(\b{z}_{i,j}\, |\, \b{x}_i, \b{\theta}_e)$ and $\mathbb{P}(\b{z}_{i,j})$, the first term in Eq. (\ref{equation_VAE_simplification2}) can be computed analytically. In the following, we simply Eq. (\ref{equation_VAE_simplification2}) further for Gaussian distributions. 

% \hfill\break
% \textbf{Simplification Type 2 for Special Case of Gaussian Distributions:}

\subsubsection{Simplification Type 2 for Special Case of Gaussian Distributions}

We can compute the KL divergence in the first term of Eq. (\ref{equation_VAE_simplification2}) analytically for univariate or multivariate Gaussian distributions. For this, we need two following lemmas. 

% https://stats.stackexchange.com/questions/234757/how-to-use-kullback-leibler-divergence-if-mean-and-standard-deviation-of-of-two
% https://stats.stackexchange.com/questions/7440/kl-divergence-between-two-univariate-gaussians
\begin{lemma}\label{lemma_KL_univariate_Gaussian}
The KL divergence between two univariate Gaussian distributions $p_1 \sim \mathcal{N}(\mu_1, \sigma_1^2)$ and $p_2 \sim \mathcal{N}(\mu_2, \sigma_2^2)$ is:
\begin{align}
\text{KL}(p_1 \| p_2) = \log(\frac{\sigma_2}{\sigma_1}) + \frac{\sigma_1^2 + (\mu_1 - \mu_2)^2}{2 \sigma_2^2} - \frac{1}{2}.
\end{align}
\end{lemma}
\begin{proof}
See Appendix \ref{section_appendix_A} for proof.
\end{proof}

% http://stanford.edu/~jduchi/projects/general_notes.pdf
% https://stats.stackexchange.com/questions/234757/how-to-use-kullback-leibler-divergence-if-mean-and-standard-deviation-of-of-two
% https://stats.stackexchange.com/questions/60680/kl-divergence-between-two-multivariate-gaussians
\begin{lemma}\label{lemma_KL_multivariate_Gaussian}
The KL divergence between two multivariate Gaussian distributions $p_1 \sim \mathcal{N}(\b{\mu}_1, \b{\Sigma}_1)$ and $p_2 \sim \mathcal{N}(\b{\mu}_2, \b{\Sigma}_2)$ with dimensionality $p$ is:
\begin{align}
\text{KL}(p_1 \| p_2) &= \frac{1}{2} \Big( \log(\frac{|\b{\Sigma}_2|}{|\b{\Sigma}_1|}) - p + \textbf{tr}(\b{\Sigma}_2^{-1} \b{\Sigma}_1) \nonumber \\
&~~~~~ + (\b{\mu}_2 - \b{\mu}_1)^\top \b{\Sigma}_2^{-1} (\b{\mu}_2 - \b{\mu}_1) \Big).
\end{align}
\end{lemma}
\begin{proof}
See {\citep[Section 9]{duchi2007derivations}} for proof.
\end{proof}

Consider the case in which we have:
\begin{align}
& \mathbb{P}(\b{z}_{i}\, |\, \b{x}_i, \b{\theta}_e) \sim \mathcal{N}(\b{\mu}_{z|x}, \b{\Sigma}_{z|x}), \\
& \mathbb{P}(\b{z}_{i}) \sim \mathcal{N}(\b{\mu}_{z}, \b{\Sigma}_{z}),
\end{align}
where $\b{z}_i \in \mathbb{R}^p$.
Note that the parameters $\b{\mu}_{z|x}$ and $\b{\Sigma}_{z|x}$ are trained in neural network while the parameters $\mathbb{P}(\b{z}_{i,j})$ can be set to $\b{\mu}_{z} = \b{0}$ and $\b{\Sigma}_{z} = \b{I}$ inspired by Eq. (\ref{equation_factor_analysis_prior_z}) in factor analysis. 
According to Lemma \ref{lemma_KL_multivariate_Gaussian}, the approximation of ELBO, i.e. Eq. (\ref{equation_VAE_simplification2}), can be simplified to:
\begin{align}
&\sum_{i=1}^b \mathcal{L}(q, \b{\theta}) \approx \sum_{i=1}^b \widetilde{\mathcal{L}}(q, \b{\theta}) \nonumber \\
&~~~~~~~~~~~~~ = - \sum_{i=1}^b \frac{1}{2} \Big( \log(\frac{|\b{\Sigma}_z|}{|\b{\Sigma}_{z|x}|}) - p + \textbf{tr}(\b{\Sigma}_z^{-1} \b{\Sigma}_{z|x}) \nonumber \\
&~~~~~~~~~~~~~~~~~ + (\b{\mu}_z - \b{\mu}_{z|x})^\top \b{\Sigma}_z^{-1} (\b{\mu}_z - \b{\mu}_{z|x}) \Big) \nonumber \\
&~~~~~~~~~~~~~~~~~ + \sum_{i=1}^b \frac{1}{\ell} \sum_{j=1}^{\ell} \log\big(\mathbb{P}(\b{x}_i\, |\, \b{z}_{i,j}, \b{\theta}_d)\big).
\end{align}

\subsubsection{Training Variational Autoencoder with Approximations}

We can train VAE with EM, where Monte Carlo approximations are applied to ELBO. The Eqs. (\ref{equation_VAE_E_step_gradient}) and (\ref{equation_VAE_M_step_gradient}) are replaced by the following equations:
\begin{align}
&\text{E-step:} ~~~~~ \b{\theta}_e^{(t)} := \b{\theta}_e^{(t-1)} + \eta_e \frac{\partial \sum_{i=1}^b \widetilde{\mathcal{L}}(\b{\theta}_e, \b{\theta}_d^{(t-1)}, \b{x}_i)}{\partial \b{\theta}_e}, \label{equation_VAE_E_step_gradient_approx} \\
&\text{M-step:} ~~~~~ \b{\theta}_d^{(t)} := \b{\theta}_d^{(t-1)} + \eta_d \frac{\partial \sum_{i=1}^b \widetilde{\mathcal{L}}(\b{\theta}_e^{(t)}, \b{\theta}_d, \b{x}_i)}{\partial \b{\theta}_d}, \label{equation_VAE_M_step_gradient_approx}
\end{align}
where the approximated ELBO was introduced in previous sections.

\subsubsection{Prior Regularization}

Some works regularize the ELBO, Eq. (\ref{equation_VAE_ELBO_KL_divergence}), with a penalty on the prior distribution $\mathbb{P}(\b{z}_i)$. Using this, we guide the learned distribution of latent space $\mathbb{P}(\b{z}_i\, |\, \b{x}_i, \b{\theta}_e)$ to have a specific prior distribution $\mathbb{P}(\b{z}_i)$. 
Some examples for prior regularization in VAE are geodesic priors \cite{hadjeres2017glsr} and optimal priors \cite{takahashi2019variational}.
Note that this regularization can inject domain knowledge to the latent space. It can also be useful for making the latent space more interpretable. 

\subsection{The Reparametrization Trick}

Sampling the $\ell$ samples for the latent variables, i.e. Eq. (\ref{equation_VAE_sampling}), blocks the gradient flow because computing the derivatives through $\mathbb{P}(\b{z}_i\, |\, \b{x}_i, \b{\theta}_e)$ by chain rule gives a high variance estimate of gradient. 
In order to overcome this problem, we use the reparameterization technique \cite{kingma2014auto,rezende2014stochastic,titsias2014doubly}. In this technique, instead of sampling $\b{z}_i \sim \mathbb{P}(\b{z}_i\, |\, \b{x}_i, \b{\theta}_e)$, we assume $\b{z}_i$ is a random variable but is a deterministic function of another random variable $\b{\epsilon}_i$ as follows:
\begin{align}
\b{z}_i = g(\b{\epsilon}_i, \b{x}_i, \b{\theta}_e),
\end{align}
where $\b{\epsilon}_i$ is a stochastic variable sampled from a distribution as:
\begin{align}
\b{\epsilon}_i \sim \mathbb{P}(\b{\epsilon}).
\end{align}
The Eqs. (\ref{equation_VAE_simplification1_ELBO_withExpectation}) and \ref{equation_VAE_simplification2_ELBO_withExpectation} both contain an expectation of a function $f(\b{z}_i)$. Using this technique, this expectation is replaced as:
\begin{align}
\mathbb{E}_{\sim \mathbb{P}(\b{z}_i\, |\, \b{x}_i, \b{\theta}_e)} [f(\b{z}_i)] \rightarrow \mathbb{E}_{\sim \mathbb{P}(\b{z}_i\, |\, \b{x}_i, \b{\theta}_e)} [f(g(\b{\epsilon}_i, \b{x}_i, \b{\theta}_e))].
\end{align}
Using the reparameterization technique, the encoder, which implemented $\mathbb{P}(\b{z}_i\, |\, \b{x}_i, \b{\theta}_e)$, is replaced by $g(\b{\epsilon}_i, \b{x}_i, \b{\theta}_e)$ where in the latent space between encoder and decoder, we have $\b{\epsilon}_i \sim \mathbb{P}(\b{\epsilon})$ and $\b{z}_i = g(\b{\epsilon}_i, \b{x}_i, \b{\theta}_e)$.

A simple example for the reparameterization technique is when $z_i$ and $\epsilon_i$ are univariate Gaussian variables:
\begin{align*}
&z_i \sim \mathcal{N}(\mu, \sigma^2), \\
&\epsilon_i \sim \mathcal{N}(0, 1), \\
&z_i = g(\epsilon_i) = \mu + \sigma \epsilon_i.
\end{align*}

For some more advanced reparameterization techniques, the reader can refer to \cite{figurnov2018implicit}.

\subsection{Training Variational Autoencoder with Backpropagation}

In practice, VAE is trained by backpropagation \cite{rezende2014stochastic} where the backpropagation algorithm \cite{rumelhart1986learning} is used for training the weights of network. Recall that in training VAE with EM, the encoder and decoder are trained separately using the E-step and the M-step of EM, respectively. However, in training VAE with backpropagation, the whole network is trained together and not in separate steps. Suppose the whole weights if VAE are denoted by $\b{\theta} := \{\b{\theta}_e, \b{\theta}_d\}$. Backpropagation trains VAE using the mini-batch stochastic gradient descent with the negative ELBO, $\sum_{i=1}^b -\widetilde{\mathcal{L}}(\b{\theta}, \b{x}_i)$, as the loss function:
\begin{align}
\b{\theta}^{(t)} := \b{\theta}^{(t-1)} - \eta \frac{\partial \sum_{i=1}^b -\widetilde{\mathcal{L}}(\b{\theta}, \b{x}_i)}{\partial \b{\theta}},
\end{align}
where $\eta$ is the learning rate. Note that we are minimizing here because neural networks usually minimize the loss function. 

\subsection{The Test Phase in Variational Autoencoder}

In the test phase, we feed the test data point $\b{x}_i$ to the encoder to determine the parameters of the conditional distribution of latent space, i.e., $\mathbb{P}(\b{z}_i\, |\, \b{x}_i, \b{\theta}_e)$. Then, from this distribution, we sample the latent variable $\b{z}_i$ from the latent space and generate the corresponding reconstructed data point $\b{x}_i$ by the decoder. As you see, VAE is a generative model which generates data points \cite{ng2002discriminative}.

\subsection{Other Notes and Other Variants of Variational Autoencoder}

There exist many improvements on VAE. Here, we briefly review some of these works. One of the problems of VAE is generating blurry images when data points are images. This blurry artifact may be because of several following reasons:
\begin{itemize}
\item sampling for the Monte Carlo approximations
\item lower bound approximation by ELBO
\item restrictions on the family of distributions where usually simple Gaussin distributions are used.
\end{itemize}
There are some other interpretations for the reason of this problem; for example, see \cite{zhao2017towards}.
This work also proposed a generalized ELBO. 
Note that generative adversarial networks \cite{goodfellow2014generative} usually generate clearer images; therefore, some works have combined variational and adversarial inferences \cite{mescheder2017adversarial} for using the advantages of both models. 

Variational discriminant analysis \cite{yu2020variational} has also been proposed for classification and discrimination of classes. 
Two other tutorials on VAE are \cite{doersch2016tutorial} and \cite{odaibo2019tutorial}. 
Some more recently published papers on VAE are nearly optimal VAE \cite{bai2020nearly}, deep VAE \cite{hou2017deep}, Hamiltonian VAE \cite{caterini2018hamiltonian}, and Nouveau VAE \cite{vahdat2020nvae} which is a hierarchical VAE. For image data and image and caption modeling, a fusion of VAE and convolutional neural network is also proposed \cite{pu2016variational}. 
The influential factors in VAE are also analyzed in the paper \cite{liu2020discovering}. 

\section{Conclusion}\label{section_conclusion}

This paper was a tutorial and survey on several dimensionality reduction and generative model which are tightly related. Factor analysis, probabilistic PCA, variational inference, and variational autoencoder are covered in this paper. All of these methods assume that every data point is generated from a latent variable or factor where some noise have also been applied on data in the data space. 

\section*{Acknowledgement}

The authors hugely thank the instructors of deep learning course at the Carnegie Mellon University (you can see their YouTube channel) whose lectures partly covered some materials mentioned in this tutorial paper. 

\appendix

\section{Proof for Lemma \ref{lemma_KL_univariate_Gaussian}}\label{section_appendix_A}

% https://stats.stackexchange.com/questions/7440/kl-divergence-between-two-univariate-gaussians

\begin{align*}
&\text{KL}(p_1\|p_2) = \int p_1(x) \log\big(\frac{p_1(x)}{p_2(x)}\big) dx \\
&=\int p_1(x) \log(p_1(x))\, dx - \int p_1(x) \log(p_2(x))\, dx.
\end{align*}
According to integration by parts, we have:
\begin{align*}
\int p_1(x) \log(p_1(x))\, dx = -\frac{1}{2} (1 + \log (2 \pi \sigma_1^2)).
\end{align*}
We also have:
\begin{align*}
&-\int p_1(x) \log(p_2(x))\, dx \\
&= -\int p_1(x) \log\Big( \frac{1}{\sqrt{2 \pi \sigma_2^2}} e^{-\frac{(x - \mu_2)^2}{2 \sigma_2^2}} \Big)\, dx \\
&= \frac{1}{2} \log( 2 \pi \sigma_2^2 ) \underbrace{\int p_1(x) dx}_{=1} -\int p_1(x) \log\big( e^{-\frac{(x - \mu_2)^2}{2 \sigma_2^2}} \big)\, dx \\
&= \frac{1}{2} \log( 2 \pi \sigma_2^2 ) +\int p_1(x) \frac{(x - \mu_2)^2}{2 \sigma_2^2}\, dx \\
&= \frac{1}{2} \log( 2 \pi \sigma_2^2 ) + \frac{1}{2 \sigma_2^2} \Big(\int p_1(x) x^2 dx \\
&~~~~~~~~~~~~~~~~~~~~~~~ - \int p_1(x) 2x\mu_2 dx + \int p_1(x) \mu_2^2 dx\Big) \\
&= \frac{1}{2} \log( 2 \pi \sigma_2^2 ) \\
&~~~~~~~~~~~~~~~~ + \frac{1}{2 \sigma_2^2} \Big(\mathbb{E}_{\sim p_1(x)} [x^2] - 2 \mu_2 \mathbb{E}_{\sim p_1(x)} [x] + \mu_2^2 \Big).
\end{align*}
We know that:
\begin{align*}
\mathbb{V}\text{ar}[x] = \mathbb{E}[x^2] - \mathbb{E}[x]^2 \implies \mathbb{E}[x^2] = \sigma_1^2 + \mu_1^2
\end{align*}
Hence:
\begin{align*}
&-\int p_1(x) \log(p_2(x))\, dx \\
&= \frac{1}{2} \log( 2 \pi \sigma_2^2 ) + \frac{1}{2 \sigma_2^2} \Big(\sigma_1^2 + \mu_1^2 - 2 \mu_2 \mu_1 + \mu_2^2 \Big) \\
&= \frac{1}{2} \log( 2 \pi \sigma_2^2 ) + \frac{1}{2 \sigma_2^2} \big( \sigma_1^2 + (\mu_1 - \mu_2) \big).
\end{align*}
Therefore, finally, we have:
\begin{align*}
\text{KL}(p_1\|p_2) &= -\frac{1}{2} (1 + \log (2 \pi \sigma_1^2))  \\
&+ \frac{1}{2} \log( 2 \pi \sigma_2^2 ) + \frac{1}{2 \sigma_2^2} \big( \sigma_1^2 + (\mu_1 - \mu_2) \big) \\
&= \log(\frac{\sigma_2}{\sigma_1}) + \frac{\sigma_1^2 + (\mu_1 - \mu_2)^2}{2 \sigma_2^2} - \frac{1}{2}.
\end{align*}
Q.E.D.

% \section{Proof for Lemma \ref{lemma_KL_multivariate_Gaussian}}\label{section_appendix_B}

% % https://stats.stackexchange.com/questions/60680/kl-divergence-between-two-multivariate-gaussians

% % http://stanford.edu/~jduchi/projects/general_notes.pdf

\bibliography{References}
\bibliographystyle{icml2016}

\end{document}